\documentclass[11pt,reqno]{amsart}
\usepackage{graphicx,verbatim,lineno,titletoc}
\usepackage{amssymb,mathrsfs}
\usepackage{amsmath}
\usepackage{adjustbox}
\usepackage{calc}
\usepackage{ytableau}
\usepackage{array}
\usepackage{tikz}
\usetikzlibrary{shapes.multipart,patterns,arrows}
\usepackage[font=footnotesize]{caption}
\usepackage[T1]{fontenc} 
\usepackage{fourier} 
\usepackage[english]{babel} 
\usepackage{appendix}
\usepackage[all]{xy}
\usepackage{enumerate}
\usepackage[english]{babel}
\usepackage{multirow}
\usepackage{float}
\usepackage{enumitem}
\usepackage{accents}
\usepackage[numbers]{natbib}
\usepackage{hyperref}
\hypersetup{colorlinks}
\usepackage{algorithm}
\usepackage[noend]{algpseudocode}
\usepackage{mathtools}
\mathtoolsset{showonlyrefs}

\newcommand\dataset[1]{\textsc{\texttt{#1}}}
 
\usepackage[top=1.0in, left=1.1in, right=1.1in, bottom=1.05in]{geometry}
\hypersetup{colorlinks,linkcolor={red},citecolor={olive},urlcolor={red}}

\newcommand{\edit}[1]{{\textcolor{black}{#1}}}

\newcolumntype{M}[1]{>{\centering\arraybackslash}m{#1}}
\newcolumntype{N}{@{}m{0pt}@{}}

\pdfoutput=1

\newtheorem{theorem}{Theorem}[section]
\newtheorem{prop}[theorem]{Proposition}

\newtheorem{lemma}[theorem]{Lemma}
\newtheorem{corollary}[theorem]{Corollary}

\newtheorem{remark}[theorem]{Remark}

\def\limsup{\mathop{\rm lim\,sup}\limits}

\def\R{\mathbb{R}}

\def\E{\mathbb{E}}
\def\P{\mathbb{P}}

\def\eps{\varepsilon}

\def\X{\mathbf{X}}
\def\x{\mathbf{x}}

\def\G{\mathcal{G}}

\newcommand{\tr}{\textup{tr}}

\DeclareMathOperator*{\argmin}{arg\,min}

\newlength\myindent
\newtheorem{innercustomassumption}{}
\newenvironment{customassumption}[1]
{\renewcommand\theinnercustomassumption{#1}\innercustomassumption}
{\endinnercustomassumption}

\theoremstyle{definition}
\newtheorem{innercustomexample}{Example}
\newenvironment{customexample}[1]
{\renewcommand\theinnercustomexample{#1}\innercustomexample}
{\endinnercustomexample}

\allowdisplaybreaks

\makeatletter
\newcommand{\addresseshere}{%
	\enddoc@text\let\enddoc@text\relax
}
\makeatother

\begin{document}

\title[OMF for Markovian data and network dictionary learning]{Online matrix factorization for Markovian data \\ and applications to Network Dictionary Learning}

\author{Hanbaek Lyu}
\address{Hanbaek Lyu, Department of Mathematics, University of California, Los Angeles, CA 90095, USA}
\email{\texttt{hlyu@math.ucla.edu}}

\author{Deanna Needell}
\address{Deanna Needell, Department of Mathematics, University of California, Los Angeles, CA 90095, USA}
\email{\texttt{deanna@math.ucla.edu}}

\author{Laura Balzano}
\address{Laura Balzano, Department of Electrical Engineering and Computer Science, University of Michigan, Ann Arbor, MI 48109, USA}
\email{\texttt{girasole@umich.edu}}

\thanks{The codes for the main algorithm and simulations are provided in \texttt{https://github.com/HanbaekLyu/ONMF$\_$ONTF$\_$NDL}}


\begin{abstract}
	Online Matrix Factorization (OMF) is a fundamental tool for dictionary learning problems, giving an approximate representation of complex data sets in terms of a reduced number of extracted features. Convergence guarantees for most of the OMF algorithms in the literature assume independence between data matrices, and the case of dependent data streams remains largely unexplored. In this paper, we show that a non-convex generalization of the well-known OMF algorithm for i.i.d.\ stream of data in \citep{mairal2010online} converges almost surely to the set of critical points of the expected loss function, even when the data matrices are functions of some underlying Markov chain satisfying a mild mixing condition. This allows one to extract features more efficiently from  dependent data streams, as there is no need to subsample the data sequence to approximately satisfy the independence assumption. As the main application, by combining online non-negative matrix factorization and a recent MCMC algorithm for sampling motifs from networks, we propose a novel framework of \textit{Network Dictionary Learning}, which extracts ``network dictionary patches' from a given network in an online manner that encodes main features of the network. \edit{We demonstrate this technique and its application to network denoising problems on real-world network data.}
\end{abstract}

\maketitle

\section{Introduction}
\label{Introduction}

In modern data analysis, a central step is to find a low-dimensional representation to better understand, compress, or convey the key phenomena captured in the data. Matrix factorization provides a powerful setting for one to describe data in terms of a linear combination of factors or atoms. In this setting, we have a data matrix $X \in \R^{d \times n}$, and we seek a factorization of $X$ into the product $WH$ for $W \in \R^{d \times r}$ and $H \in \R^{r \times n}$. This problem has gone by many names over the decades, each with different constraints: dictionary learning, factor analysis, topic modeling, component analysis. It has applications in text analysis, image reconstruction, medical imaging, bioinformatics, and many other scientific fields more generally \citep{sitek2002correction, berry2005email, berry2007algorithms, chen2011phoenix, taslaman2012framework, boutchko2015clustering, ren2018non}.

\vspace{-0.3cm}
\begin{figure}[H]
	\centering\hspace{-0.7cm}
	\includegraphics[width=0.8 \linewidth]{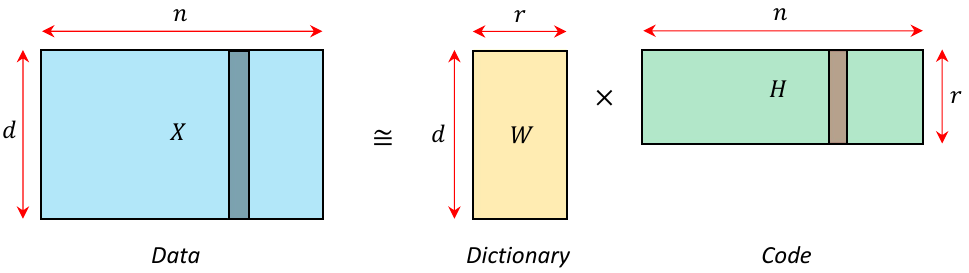}
	\vspace{-0.3cm}
	\caption{Illustration of  matrix factorization. Each column of the data matrix is approximated by the linear combination of the columns of the dictionary matrix with coefficients given by the corresponding column of the code matrix.}
	\label{fig:NMF_diagram}
\end{figure}

\vspace{-0.3cm}
Online matrix factorization (OMF) is a problem setting where data are accessed in a streaming manner and the matrix factors should be updated each time. That is, we get draws of $X$ from some distribution $\pi$ and seek the best factorization such that the expected loss $\E_{X\sim \pi}\left[ \lVert X - WH \rVert_{F}^{2} \right]$ is small. This is a relevant setting in today's data world, where large companies, scientific instruments, and healthcare systems are collecting massive amounts of data every day. One cannot compute with the entire dataset, and so we must develop online algorithms to perform the computation of interest while accessing them sequentially. There are several algorithms for computing factorizations of various kinds in an online context. Many of them have algorithmic convergence guarantees, however, \edit{all these guarantees require that data are sampled independently from a fixed  distribution.} In all of the application examples mentioned above, one may make an argument for (nearly) identical distributions (e.g., using subsampling), but never for independence. This assumption is critical to the analysis of previous works (see., e.g., \citep{mairal2010online, guan2012online, zhao2017online}). 

A natural way to relax the assumption of independence in this online context is through the Markovian assumption. In many cases one may assume that the data are not independent, but independent conditioned on the previous iteration. The central contribution of our work is to extend the analysis of online matrix factorization in \citep{mairal2010online} to the setting where the sequential data form a Markov chain. This is naturally motivated by the fact that the Markov chain Monte Carlo (MCMC) method is one of the most versatile sampling techniques across many disciplines, where one designs a Markov chain exploring the sample space that converges to the target distribution.

As the main application of our result, we propose a novel framework for network data analysis that we call \textit{Network Dictionary Learning} (see Section \ref{section:NDL}). This allows one to extract ``network dictionary atoms'' from a given network that capture the most important local subgraph patterns in the network. We also propose a network reconstruction algorithm using the learned network dictionary atoms. A key ingredient is a recent MCMC algorithm for sampling motifs from networks developed by Lyu together with Memoli and Sivakoff \citep{lyu2019sampling}, which provides a stream of correlated subgraph patterns of the given network. We provide convergence guarantee of our Network Dictionary Learning algorithm (Corollary \ref{cor:NDL}) as a corollary of our main result, and illustrate our framework through various real-world network data (see Section \ref{section:NDL}). 

\subsection{Theoretical contribution}
\label{subsection:contribution}

The main result in the present paper, Theorem \ref{thm:online NMF_convergence}, rigorously establishes convergence of a non-convex generalization \eqref{eq:scheme_online NMF_surrogate2} of the online matrix factorization scheme in \citep{mairal2010online, mairal2013stochastic} when the data sequence $(X_{t})_{t\in \mathbb{N}}$ is realized as a function \edit{$\varphi(Y_{t})$ of some underlying Markov chain $Y_{t}$ (which includes the case that $X_{t}$ itself forms a Markov chain)} with a mild mixing condition. A practical implication of our result is that one can now extract features more efficiently from dependent data streams, as there is no need to subsample the data sequence to approximately satisfy the independence assumption. We illustrate this point through an application to sequences of correlated Ising spin configurations generated by the Gibbs sampler (see Section \ref{section:Ising}). Our application to the Ising model can easily be generalized to other well-known spin systems such as the cellular Potts model in computational biology \citep{ouchi2003improving, maree2007cellular, szabo2013cellular} and the restricted Boltzmann Machine in machine learning \citep{nair2010rectified} (see Section \ref{section:Ising}). 

\edit{An important related work is \citep{mensch2017stochastic}, where the authors obtain a perturbative analysis of the original work in \citep{mairal2010online}. In the former, convergence of the OMF algorithm in \citep{mairal2010online} (see also \eqref{eq:scheme_online NMF_surrogate2}) has been established when the time-$t$ data matrix $X_{t}$ conditional on the past data is sampled from an approximate distribution $\pi_{t}$ that converges to the true distribution $\pi$ at an exponential rate (see assumption \textbf{(H)} in \citep{mensch2017stochastic}). While this work provides an important theoretical justification of the use of code approximation in the OMF algorithm in \citep{mairal2010online} (see $H_{t}$ in \eqref{eq:scheme_online NMF_surrogate2}), we emphasize that this result does \textit{not} imply our main result (Theorem \ref{thm:online NMF_convergence}) in the present work. Indeed, when the data $(X_{t})_{t\in \mathbb{N}}$ forms a Markov chain with transition matrix $P$, we have $\pi_{t}=P(X_{t-1}, \cdot)$,  and this conditional distribution can even be a constant distance away from the stationary distribution $\pi$. (For instance, consider the case when $X_{t}$ alternates between two matrices. Then $\pi=[1/2,1/2]$ and $\pi_{t}$ is either $[1,0]$ or $[0,1]$ for all $t\ge 1$.) In fact, this is the main difficulty in analyzing the OMF algorithm \eqref{eq:scheme_online NMF_surrogate2} for \textit{dependent} data sequences.  This is a nontriviality that we address in our current work here.}

The proof of our main result (Theorem \ref{thm:online NMF_convergence}) adopts a number of techniques used in \citep{mairal2010online, mairal2013stochastic} for the i.i.d. input, but uses a key innovation that handles dependence in the data matrices directly without subsampling, which can potentially be applied to relax independence assumptions for convergence of other online algorithms. The theory of quasi-martingales \citep{fisk1965quasi, rao1969quasi} is a key ingredient in convergence analysis under i.i.d input in \citep{mairal2010online,mairal2013stochastic} as well as many other related works. Namely, one shows that 
\begin{align}\label{eq:def_quasi_martingale}
\sum_{t=0}^{\infty} \left( \E\left[  \hat{f}_{t+1}(W_{t+1}) - \hat{f}_{t}(W_{t})   \,\bigg|\, \mathcal{F}_{t} \right]\right)^{+}  < \infty,
\end{align}	
where $\hat{f}_{t}$ denotes an associated surrogate loss function, $W_{t}$ the learned dictionary at time $t$, and $\mathcal{F}_{t}$ the filtration of the information up to time $t$. However, this is not necessarily true without the independence assumption. Our key insight to overcome this issue is that, while the 1-step conditional distribution $P(X_{t-1}, \cdot)$ may be far from the stationary distribution $\pi$, the $N$-step conditional distribution $P^{N}(X_{t-N}, \cdot)$ is exponentially close to $\pi$ under mild conditions. More precisely, we use conditioning on a ``distant past'' $\mathcal{F}_{t-a_{t}}$, not on the present $\mathcal{F}_{t}$, in order to allow the Markov chain to mix close enough to the stationary distribution $\pi$ for $a_{t}$ iterations. Then concentration of Markov chains allows us to choose a suitable sequence $1\le a_{t} \le t$ (see Proposition \ref{prop:increment_bd} and Lemma \ref{lemma:increment_bd}), for which we show  
\begin{align}\label{eq:def_quasi_martingale0}
\sum_{t=0}^{\infty} \E\left[ \left( \E\left[  \hat{f}_{t+1}(W_{t+1}) - \hat{f}_{t}(W_{t})   \,\bigg|\, \mathcal{F}_{t-a_{t}} \right]\right)^{+}\right]  < \infty
\end{align}	
in the dependent case, from which we derive our main result.

\section{Background}

In this section, we provide some relevant background and state the main problem and algorithm. 
\subsection{Topic modeling and matrix factorization}

\textit{Topic modeling} (or \textit{dictionary learning}) aims at extracting important features of a complex dataset so that one can approximately represent the dataset in terms of a reduced number of extracted features (topics) \citep{blei2003latent}. Topic models have been shown to efficiently capture latent intrinsic structures of text data in natural language processing tasks \citep{steyvers2007probabilistic, blei2010probabilistic}. One of the advantages of topic modeling based approaches is that the extracted topics are often directly interpretable, as opposed to the arbitrary abstraction of deep neural network based approach.

\textit{Matrix factorization} is one of the fundamental tools in dictionary learning problems. Given a large data matrix $X$, can we find some small number of ``dictionary vectors'' so that \edit{we can represent each column of the data matrix as a linear combination of dictionary vectors?} More precisely, given a data matrix $X\in \mathbb{R}^{d\times n}$ and sets of admissible factors $\mathcal{C}\subseteq \mathbb{R}^{d\times r}$ and $\mathcal{C}'\subseteq \mathbb{R}^{r\times n}$, we wish to factorize $X$ into the product of $W\in \mathcal{C}$ and $H\in \mathcal{C'}$ by solving the following optimization problem
\begin{align}\label{eq:NMF_error1}
\inf_{W\in \mathcal{C}\subseteq \mathbb{R}^{d\times r},\, H\in \mathcal{C}'\subseteq \mathbb{R}^{r\times n} }  \lVert X - WH  \rVert_{F}^{2},
\end{align}
where $\lVert A \rVert_{F}^{2} = \sum_{i,j} A_{ij}^{2}$ denotes the  Frobenius norm. Here $W$ is called the \textit{dictionary} and $H$ is the \textit{code} of data $X$ using dictionary $W$. A solution of such matrix factorization problem is illustrated in Figure \ref{fig:NMF_diagram}.

When there are no constraints for the dictionary and code matrices, i.e., $\mathcal{C}=\mathbb{R}^{d\times r}$ and $\mathcal{C}'=\mathbb{R}^{r\times n}$, then the optimization problem \eqref{eq:NMF_error1} is equivalent to \textit{principal component analysis}, which is one of the primary techniques in data compression and dictionary learning. In this case, the optimal dictionary $W$ for $X$ is given by the top $r$ eigenvectors of its covariance matrix, and the corresponding code $H$ is obtained by projecting $X$ onto the subspace generated by these eigenvectors. However, the dictionary vectors found in this way are often hard to interpret. This is in part due to the possible cancellation between them when we take their linear combination, with both positive and negative coefficients.

When the admissible factors are required to be non-negative, the optimization problem \eqref{eq:NMF_error1} is an instance of \textit{Nonnegative matrix factorization} (NMF), which is one of the fundamental tools in dictionary learning problems that provides a parts-based representation of high dimensional data \citep{lee1999learning, lee2009semi}. Due to the non-negativity constraint, each column of the data matrix is then represented as a non-negative linear combination of dictionary elements (See Figure \ref{fig:NMF_diagram}). Hence the dictionaries must be "positive parts" of the columns of the data matrix. When each column consists of a human face image, NMF learns the parts of human face (e.g., eyes, nose, and mouth). This is in contrast to principal component analysis and vector quantization: Due to cancellation between eigenvectors, each ``eigenface'' does not have to be parts of face \citep{lee1999learning}.

\subsection{Online Matrix Factorization}
\label{subsection:block_opt}

Many iterative algorithms to find approximate solutions $WH$ to the optimization problem \eqref{eq:NMF_error1}, including the well-known Multiplicative Update by Lee and Seung \citep{lee2001algorithms},  are based on a block optimization scheme (see \citep{gillis2014and} for a survey). Namely, we first compute its representation $H_{t}$ using the previously learned dictionary $W_{t-1}$, and then find an improved dictionary $W_{t}$ (see Figure \ref{fig:iterative_NMF} with setting $X_{t}\equiv X$).

\begin{figure*}[h]
	\centering
	\includegraphics[width=0.9 \linewidth]{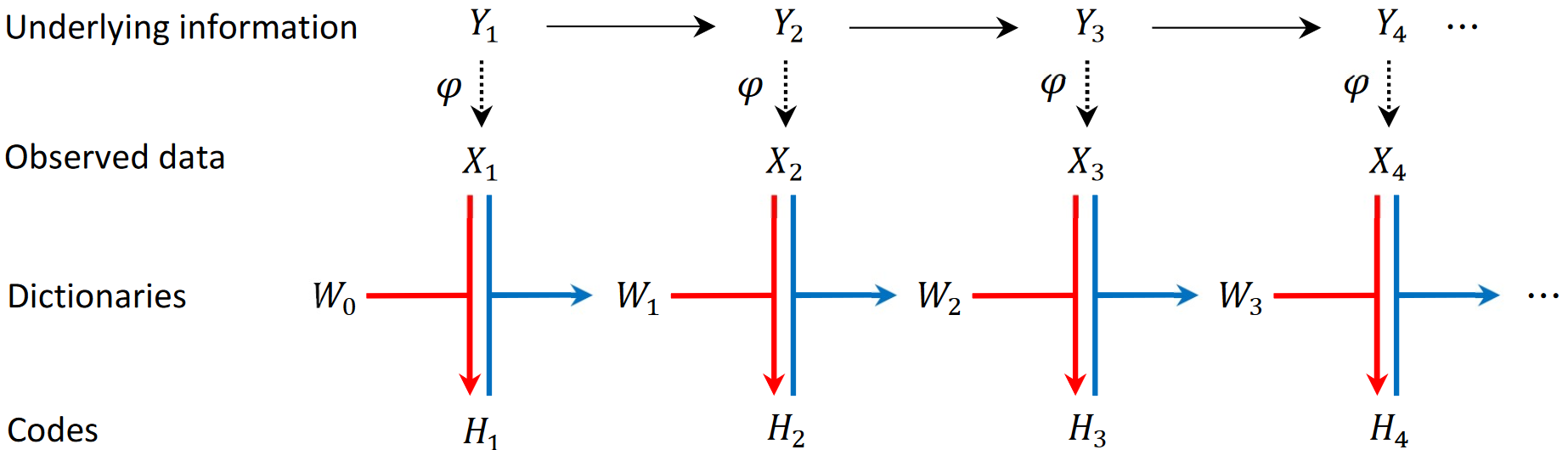}
	\caption{ Iterative block scheme of functional OMF. $(Y_{t})_{t\in \mathbb{N}}$ is the sequence of underlying information. At each time $t$, a data matrix $X_{t}$ of interest is observed from information $Y_{t}$ via a fixed function $\varphi$. Sequences of dictionaries $(W_{t})_{t\in \mathbb{N}}$ and codes $(H_{t})_{t\in \mathbb{N}}$ are learned from data matrices by a block optimization scheme.
	}
	\label{fig:iterative_NMF}
\end{figure*}

\edit{A main example in this setting is when $Y_{t}$ is a homomorphism $\x:F\rightarrow \G$ from a $k$-node network $F$ with node set $[k]=\{1,2,\dots,k\}$ into a $n$-node network $\G$ with node set $V$, which is an element of $V^{[k]}$, and $X_{t}$ is the $k\times k$ matrix that is the adjacency matrix of the $k$-node subnetwork of $\G$ induced by $\x$. There, one may sample a random homomorphism $\x$ using a Markov chain Monte Carlo (MCMC) algorithm so that $\{Y_{t}\}$ forms a Markov chain, but the induced $k\times k$ maxtrix sequence $X_{t}$ may not form a Markov chain (see Section \ref{section:NDL}).  }

Despite their popularity in dictionary learning and image processing, one of the drawbacks of these standard iterative algorithms for NMF is that we need to store the data matrix (which is of size $O(dn)$) during the iteration, so they become less practical when there is a memory constraint and yet the size of data matrix is large. Furthermore, in practice only a random sample of the entire dataset is often available, in which case we are not able to apply iterative algorithms that require the entire dataset for each iteration.


The \textit{Online Matrix Factorization} (OMF) problem concerns a similar matrix factorization problem for a \textit{sequence} of input matrices. Here we give a more general and flexible formulation of OMF. Roughly speaking, for each time $t\in \mathbb{N}$, one observes information $Y_{t}$, from which a data matrix $X_{t}$ of interest is extracted as $X_{t}=\varphi(Y_{t})$, for a fixed function $\varphi$. We then want to learn sequences of dictionaries $(W_{t})_{t\in \mathbb{N}}$ and codes $(H_{t})_{t\in \mathbb{N}}$ from the stream $(X_{t})_{ t\in \mathbb{N}}$ of data matrices. 

For a precise formulation, let $(Y_{t})_{ t\in \mathbb{N}}$ be a discrete-time stochastic process of information taking values in a fixed sample space $\Omega$ with unique stationary distribution $\pi$. Fix a function $\varphi:\Omega\rightarrow \mathbb{R}^{d\times n}$, and define $X_{t}=\varphi(Y_{t})$ for each $t\in \mathbb{N}$. Fix sets of admissible factors $\mathcal{C}\subseteq \mathbb{R}^{d\times r}$ and $\mathcal{C}'\subseteq \mathbb{R}^{r\times n}$ for the dictionaries and codes, respectively. The goal of the functional OMF problem is to construct a sequence $(W_{t},H_{t})_{t\ge 1}$ of dictionary $W_{t}\in \mathcal{C}\subseteq \mathbb{R}^{r\times d}$ and codes $H_{t}\in \mathcal{C}'\subseteq \mathbb{R}^{r\times n}$ such that, almost surely as $t\rightarrow \infty$,  
\begin{align}\label{eq:def_online NMF_problem}
\lVert X_{t} - W_{t-1}H_{t} \rVert_{F}^{2} \rightarrow  \inf_{W\in \mathcal{C},\, H\in \mathcal{C}' }  \E_{Y\sim \pi}\left[ \lVert \varphi(Y) - WH \rVert_{F}^{2} \right].
\end{align}
Here and throughout, we write $\mathbb{E}_{Y\sim \pi}$ to denote the expected value with respect to the random variable $Y$ that has the distribution described by $\pi$.  Thus, we ask that the sequence of dictionary and code pairs provides a factorization error that converges to the best case average error. Since \eqref{eq:def_online NMF_problem} is a non-convex optimization problem, it is reasonable to expect that $W_{t}$ converges only to a locally optimal solution in general. Convergence guarantees to global optimum is a subject of future work.  

We also mention recent related work in the online dictionary learning setting with the following modeling assumption: The time-$t$ data is given by $X_{t}=W^{*} H_{t}$ for some unknown but fixed dictionary $W^{*}$, and the code matrix $H_{t}$ is sampled independently from a distribution concentrated around an unknown code matrix $H^{*}$. Under some suitable additional assumptions, a convergence guarantee both for the dictionaries $W_{t}\rightarrow W^{*}$ and codes $H_{t}\rightarrow H^{*}$ has been obtained by \citet{rambhatla2019noodl}. 



\subsection{Algorithm for online matrix factorization}

In the literature of OMF, one of the crucial assumptions is that the sequence of data matrices $(X_{t})_{t\in \mathbb{N}}$ are drawn independently from a common distribution $\pi$ (see., e.g., \citep{mairal2010online, guan2012online, zhao2016online}). In this paper, we analyze convergence properties of the following scheme of OMF:
\begin{align}\label{eq:scheme_online NMF_surrogate2}
\textit{Upon arrival of $X_{t}$:\qquad }
\begin{cases}
H_{t} = \argmin_{H\in \mathcal{C}'\subseteq \R^{r\times n}} \lVert X_{t} - W_{t-1}H \rVert_{F}^{2} + \lambda \lVert H \rVert_{1}\\
A_{t} = (1-w_{t})A_{t-1}+w_{t}H_{t}H_{t}^{T} \\
B_{t} = (1-w_{t})B_{t-1}+w_{t}H_{t}X_{t}^{T} \\
W_{t} = \argmin_{W\in \mathcal{C} \subseteq \R^{d\times r}} \left(  \tr(W A_{t} W^{T})  - 2\tr(W B_{t})\right) \\
\hspace{0.85cm} \text{s.t.} \,\, \tr( (B_{t}^{T} - WA_{t} ) (W_{t-1}-W)^{T}) \le 0
\end{cases}
,
\end{align}
where $(w_{t})_{t\ge 1}$ is a prescribed sequence of weights, and $A_{0}$ and $B_{0}$ are zero matrices of size $r\times r$ and $r\times d$, respectively. Note that the $L_{2}$-loss function is augmented with the $\ell_{1}$-regularization term $\lambda \lVert H \rVert_{1}$ with regularization parameter $\lambda\ge 0$, which forces the code $H_{t}$ to be sparse. See Appendix \ref{subsection:algorithm} for a more detailed algorithm implementing \eqref{eq:scheme_online NMF_surrogate2}.

In the above scheme, the auxiliary matrices $A_{t}\in \mathbb{R}^{r\times r}$ and $B_{t}\in \mathbb{R}^{r\times d}$ effectively aggregate the history of data matrices $X_{1},\dots,X_{t}$ and their best codes $H_{1},\dots,H_{t}$. The previous dictionary $W_{t-1}$ is updated to $W_{t}$, which minimizes a quadratic loss function $\tr(W A_{t} W^{T})  - 2\tr(W B_{t})$ in the (not necessarily convex) constraint set $\mathcal{C}\subseteq \mathbb{R}^{d\times r}$ subject to the additional ellipsoidal constraint $\tr( (B_{t}^{T} - WA_{t} ) (W_{t-1}-W)^{T}) \le 0$. This extra condition means that $W_{t}$ should lie inside an ellipsoid with an axis between the previous iterate $W_{t-1}$ and the global minimum $A_{t}^{-1}B_{t}^{T}$ of the unconstrained quadratic function (see Figure \ref{fig:ellipsoid} for illustration). We present an algorithm that solves this quadratic problem when $\mathcal{C}$ is the disjoint union of convex sets (see Algorithm \ref{algorithm:dictionary_update}).

When $\mathcal{C}$ is convex and $w_{t}=1/t$ for all $t\ge 1$, the ellipsoidal condition for $W_{t}$ becomes redundant and \eqref{eq:scheme_online NMF_surrogate2} reduces to the classical algorithm of OMF in the celebrated work by Mairal et al. \citep{mairal2010online}. Assuming that $X_{t}$'s are independently drawn from the stationary distribution $\pi$, with additional mild assumption, the authors of \citet{mairal2010online} proved that the sequence $(W_{t})_{t\in \mathbb{N}}$ converges to a critical point of the expected loss function in \eqref{eq:def_online NMF_problem} augmented with the $\ell_{1}$-regularization term $\lambda \lVert H \rVert_{1}$. Later, Mairal generalized a similar convergence result under independence assumption for a broader class of non-convex objective functions \citep{mairal2013stochastic}.

\subsection{Preliminary example of dictionary learning from dependent data samples from images}
\label{subsection:image_reconstruction}

Here we give a preliminary example of dictionary learning from dependent data samples from images. One of the well-known applications of NMF is for learning dictionary patches from images and image reconstruction. For a standard NMF application, we first choose an appropriate patch size $k\ge 1$ and extract all $k\times k$ image patches from a given image. In terms of matrices, this is to consider the set of all $(k\times k)$ submatrices of the image with consecutive rows and columns. If there are $N$ such image patches, we are forming $(k^{2}\times N)$ patch matrix to which we apply NMF to extract dictionary patches. It is reasonable to believe that there are some fundamental features in the space of all image patches since nearby pixels in the image are likely to be spatially correlated.

\begin{figure*}[h]
	\centering
	\includegraphics[width=1 \linewidth]{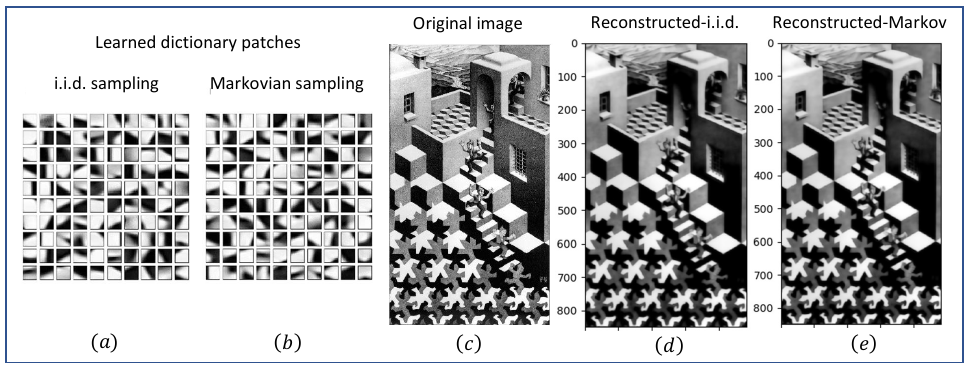}
	\vspace{-0.5cm}
	\caption{Learning 100 dictionary patches of size 10 from M.C. Escher's \textit{Cycle} (1938) shown in (c) by online NMF and two different sampling methods. Randomly sampled minibatches of 1000 patches of size 10 are fed into an online NMF algorithm for 500 iterations, where in (a), patches are sampled independently and uniformly, and in (b), the top left corner of the patches performs a simple symmetric random walk on the image. (d) and (e) show reconstructed images using the learned dictionary patches in (a) and (b), respectively. 
	}
	\label{fig:esher_renoir}
\end{figure*}

A computationally more efficient way of learning dictionary patches, especially from large images, is to use online NMF algorithms to minibatches of patches sampled randomly from the image. It is easy to sample $k\times k$ patches independently and uniformly from a given image, hence we can generate i.i.d. minibatches of a fixed number of patches. On the other hand, we can also generate a dependent sample of patches by using simple symmetric random walk on the image, meaning that the next patch is chosen from the four single-pixel shifts of the previous one with equal probability (using periodic boundary condition). A toy example for this application of online NMF for these two different sampling methods is shown in Figure \ref{fig:esher_renoir}. However, we report that the set of learned dictionaries for the random walk sampling method does appear more localized when for less iterations, which is reasonable since then the random walk may only explore a restricted portion of the entire image. See Section \ref{section:Ising} for more details about applications on dictionary learning from MCMC trajectories.

\subsection{Notation} 
Fix integers $m,n\ge 1$. We denote by $\mathbb{R}^{m\times n}$ the set of all $m\times n$ matrices of real entries. For any matrix $A$, we denote its $(i,j)$ entry, $i$th row, and $j$th column by $A_{ij}$, $[A]_{i\bullet}$, and $[A]_{\bullet j}$. For each $A=(A_{ij})\in \mathbb{R}^{m\times n}$, denote \edit{its one, Frobenius and operator norms by} $\lVert A \rVert_{1}$, $\lVert A \rVert_{F}$, and $\lVert A \rVert_{\textup{op}}$, respectively, where 
\begin{align}
\lVert A \rVert_{1}=\sum_{ij} |a_{ij}|,\quad	\lVert A \rVert_{F}^{2} = \sum_{ij}a_{ij}^{2}, \quad \lVert A \rVert_{\textup{op}} = \inf\{ c>0\,:\, \text{$\lVert Ax \rVert_{F} \le c \lVert x \rVert_{F}$ for all $x\in \mathbb{R}^{n}$} \}.
\end{align}

For any subset $\mathcal{A}\subset \mathbb{R}^{m\times n}$ and $X\in \mathbb{R}^{n\times m}$, denote 
\begin{align}\label{eq:def_radius_distance}
R(\mathcal{A}) = \sup_{X\in \mathcal{A}} \lVert X \rVert_{F}, \qquad d_{F}(X,\mathcal{A}) = \inf_{Y\in \mathcal{A}} \lVert X-Y \rVert_{F}.
\end{align}
For any continuous functional $f:\mathbb{R}^{m\times n}\rightarrow \mathbb{R}$ and a subset $A\subseteq \mathbb{R}^{N}$, we denote 
\begin{align}
\argmin_{x\in A} f = \left\{ x\in A\,\bigg|\, f(x) = \inf_{y\in A} f(y) \right\}.
\end{align} 
When $\argmin_{x\in A} f$ is a singleton $\{x^{*}\}$, we identify $\argmin_{x\in A} f$ as $x^{*}$.

For any event $A$, we let $\mathbf{1}_{A}$ denote the indicator function of $A$, where $\mathbf{1}_{A}(\omega)=1$ if $\omega\in A$ and 0 otherwise. We also denote $\mathbf{1}_{A}=\mathbf{1}(A)$ when convenient. For each $x\in \mathbb{R}$, denote $x^{+}=\max(0,x)$ and $x^{-} = \max(0,-x)$. Note that $x=x^{+}-x^{-}$ for all $x\in \mathbb{R}$ and the functions $x\mapsto x^{\pm}$ are convex.

Let $\mathbb{N}=\{0,1,2,\dots\}$ denote the set of nonnegative integers. For each integer $n\ge 1$, denote $[n]=\{1,2,\dots,n \}$. A \textit{simple graph} $G=([n], A_{G})$ is a pair of its \textit{node set} $[n]$ and its \textit{adjacency matrix} $A_{G}$, where $A_{G}$ is a symmetric 0-1 matrix with zero diagonal entries. We say nodes $i$ and $j$ are \textit{adjacent} in $G$ if $A_{G}(i,j)=1$.

\vspace{0.3cm}
\section{Preliminary discussions}

\subsection{Markov chains on countable state space}
\label{subsection:MC_intro}

\edit{We note that from here on, Markov chains will be denoted as $Y_{t}$ and we reserve $X_{t}$ to denote the data.} We first give a brief summary on Markov chains. (see, e.g., \citep{levin2017markov}). Fix a countable set $\Omega$. A function $P:\Omega^{2} \rightarrow [0,\infty)$ is called a \textit{Markov transition matrix} if every row of $P$ sums to 1. A sequence of $\Omega$-valued random variables $(Y_{t})_{t\in \mathbb{N}}$ is called a \textit{Markov chain} with transition matrix $P$ if for all $y_{0},y_{1},\dots,y_{n}\in \Omega$, 
\begin{align}\label{eq:def_MC}
\P(Y_{n}=y_{n}\,|\, Y_{n-1}=y_{n-1}, \dots, Y_{0}=y_{0}) = \P(Y_{n}=y_{n}\,|\, Y_{n-1}=y_{n-1}) = P(Y_{n-1},y_{n}). 
\end{align}
\edit{We say that a probability distribution $\pi$ on $\Omega$ is a \textit{stationary distribution} for the chain} $(Y_{t})_{t\in \mathbb{N}}$ if $\pi = \pi P$, that is, 
\begin{align}
\pi(x) = \sum_{y\in \Omega} \pi(y) P(y,x).
\end{align}
We say the chain $(Y_{t})_{t\in \mathbb{N}}$ is \textit{irreducible} if for any two states $x,y\in \Omega$ there exists an integer $t\in \mathbb{N}$ such that $P^{t}(x,y)>0$. For each state $x\in \Omega$, let $\mathcal{T}(x) = \{ t\ge 1\,|\, P^{t}(x,x)>0 \}$ be the set of times when it is possible for the chain to return to starting state $x$. We define the \textit{period} of $x$ by the greatest common divisor of $\mathcal{T}(x)$. We say the chain $Y_{t}$ is \textit{aperiodic} if all states have period 1. Furthermore, the chain is said to be \textit{positive recurrent} if there exists a state $x\in \Omega$
such that the expected return time of the chain to $x$ started from $x$ is finite. Then an irreducible and aperiodic Markov chain has a unique stationary distribution if and only if it is positive recurrent 
\citep[Thm 21.21]{levin2017markov}.

Given two probability distributions $\mu$ and $\nu$ on $\Omega$, we define their \textit{total variation distance} by 
\begin{align}
\edit{\lVert \mu - \nu \rVert_{TV} = \sup_{A\subseteq \Omega} |\mu(A)-\nu(A)|.}
\end{align}
If a Markov chain $(Y_{t})_{t\in \mathbb{N}}$ with transition matrix $P$ starts at $y_{0}\in \Omega$, then by \eqref{eq:def_MC}, the distribution of $Y_{t}$ is given by $P^{t}(y_{0},\cdot)$. If the chain is irreducible and aperiodic with stationary distribution $\pi$, then the convergence theorem (see, e.g., \citep[Thm 21.14]{levin2017markov}) asserts that the distribution of $Y_{t}$ converges to $\pi$ in total variation distance: As $t\rightarrow \infty$,
\begin{align}\label{eq:finite_MC_convergence_thm}
\sup_{y_{0}\in \Omega} \,\lVert P^{t}(y_{0},\cdot) - \pi \rVert_{TV} \rightarrow 0.
\end{align}
See \citep[Thm 13.3.3]{meyn2012markov} for a similar convergence result for the general state space chains. When $\Omega$ is finite, then the above convergence is exponential in $t$ (see., e.g., 
\citep[Thm 4.9]{levin2017markov})). Namely, there exists constants \edit{$\xi\in (0,1)$ and $C>0$ such that for all $t\in \mathbb{N}$, }
\begin{align}\label{eq:finite_MC_convergence_thm}
\max_{y_{0}\in \Omega} \,\lVert P^{t}(y_{0},\cdot) - \pi \rVert_{TV} \le C \xi^{t}.
\end{align}
\textit{Markov chain mixing} refers to the fact that, when the above convergence theorems hold, then one can approximate the distribution of $Y_{t}$ by the stationary distribution $\pi$.


\subsection{Empirical Risk Minimization for OMF}

Define the following quadratic loss function of the dictionary $W\in \mathbb{R}^{d\times r}$ with respect to data $X\in \mathbb{R}^{d\times n}$
\begin{align}\label{eq:loss_def}
\ell(X,W) = \inf_{H\in   \mathcal{C}' \subseteq\mathbb{R}^{r\times n} }  \lVert X - WH  \rVert_{F}^{2} + \lambda \lVert H \rVert_{1},
\end{align}
where $\mathcal{C}'$ denotes the set of admissible codes and $\lambda> 0$ is a fixed $\ell_{1}$-regularization parameter. For each $W\in \mathcal{C}$ define its \textit{expected loss} by 
\begin{align}\label{eq:def_f}
f(W) = \E_{Y\sim \pi} [ \ell(\varphi(Y),W)].
\end{align} 
Suppose arbitrary sequences of data matrices  $(X_{t})_{t\in \mathbb{N}}$ is given. Fix a non-increasing sequence of weights $(w_{t})_{t\in \mathbb{N}}$ in $(0,1)$. Define the \textit{(weighted) empirical loss} $f_{t}(W)$ recursively as 
\begin{align}\label{eq:def_loss_expected_empirical}
f_{t}(W) = (1-w_{t})f_{t-1}(W) + w_{t}\ell(X_{t},W), \qquad t\ge 1, W\in \mathcal{C},
\end{align}
where we take $f_{0}\equiv 0$. Note that when we take ``balanced weights'' $w_{t}=1/t$ for all $t\ge 1$, then the weighted empirical loss function takes the usual form $f_{t}(W) = t^{-1}\sum_{s=1}^{t}\ell(X_{s},W)$, where all losses are counted evenly. For $w_{t}\gg 1/t$ (e.g., $w_{t}=t^{-3/4}$), we take the recent losses more importantly than the past ones.

Suppose the sequence of data matrices $(X_{t})_{t\ge 1}$ itself is an irreducible Markov chain on $\Omega$ with unique stationary distribution $\pi$. Note that for the balanced weights $w_{t}=1/t$, by the Markov chain ergodic theorem (see, e.g., \citep[Thm 6.2.1, Ex. 6.2.4]{Durrett} or \citep[Thm. 17.1.7]{meyn2012markov}), for each dictionary $W$, the empirical loss converges almost surely to the expected loss:
\begin{align}
\lim_{t\rightarrow \infty} f_{t}(W) = f(W) \quad \text{a.s.}
\end{align}
In fact, this almost sure convergence holds for the weighted case uniformly in $W$ varying in compact $\mathcal{C}$ (see Lemma \ref{lem:uniform_convergence_asymmetric_weights}). This observation and the block optimization scheme in Subsection \ref{subsection:block_opt} suggests the following scheme for our functional OMF proglem:
\begin{align}\label{eq:scheme_online NMF}
\text{Upon arrival of $X_{t}$:\qquad }
\begin{cases}
H_{t} = \argmin_{H\in \mathcal{C}'} \lVert X_{t} - W_{t-1}H \rVert_{F}^{2} + \lambda \lVert H \rVert_{1}\\
W_{t} = \argmin_{W\in \mathcal{C}} f_{t}(W).
\end{cases}
\end{align}

Finding $H_{t}$ in \eqref{eq:scheme_online NMF} can be done using a number of known algorithms (e.g., LARS \citep{efron2004least}, LASSO \citep{tibshirani1996regression}, and feature-sign search \citep{lee2007efficient}) in this formulation. However, there are some important issues in solving the optimization problem for $W_{t}$ in \eqref{eq:scheme_online NMF}.  Note that minimizing empirical loss to find $W_{t}$ as above is an example of \textit{empirical risk minimization} (ERM), which is a classical problem in statistical learning theory \citep{vapnik1992principles}. For i.i.d., data points, recent advances guarantees that solutions of ERM even for a class of non-convex loss functions converges to the set of  local minima of the expected loss function \citep{mei2018landscape}. However, such convergence guarantee is not known for dependent data points, and there some important computational shortcomings in the above ERM for our OMF problem. Namely, in order to compute the empirical  loss $f_{t}(W)$, we may have to store the entire history of data matrices $X_{1},\dots,X_{t}$, and we need to solve $t$ instances of optimization problem \eqref{eq:loss_def} for each summand of $f_{t}(W)$. Both of these are a significant requirement for memory and computation. These issues are addressed in the OMF scheme \eqref{eq:scheme_online NMF_surrogate2}, as we discuss in the following subsection.

\subsection{Asymptotic solution minimizing surrogate loss function} 


The idea behind the OMF scheme \eqref{eq:scheme_online NMF_surrogate2} is to solve the following approximate problem  
\begin{align}\label{eq:scheme_online NMF_surrogate}
\textit{Upon arrival of $X_{t}$:\qquad }
\begin{cases}
H_{t} = \argmin_{H\in \mathcal{C}'} \lVert X_{t} - W_{t-1}H \rVert_{F}^{2} + \lambda \lVert H \rVert_{1}\\
W_{t} = \argmin_{W\in \mathcal{C}} \hat{f}_{t}(W)
\end{cases}
\end{align}
with a given initial dictionary $W_{0}\in \mathcal{C}$, where $\hat{f}_{t}(W)$ is an upper bounding surrogate for $f_{t}(W)$ defined recursively by 
\begin{align}\label{eq:def_f_hat}
\hat{f}_{t}(W) &=  (1-w_{t}) \hat{f}_{t-1}(W) + w_{t} \left( \lVert X_{t} - W H_{t} \rVert_{F}^{2} + \lambda \lVert H_{t} \rVert_{1} \right)
\end{align} 
with $\hat{f}_{0}\equiv 0$. Namely, we recycle the previously found codes $H_{1},\dots,H_{t}$ and use them as approximate solutions of the sub-problem \eqref{eq:loss_def}. Hence, there is only a single optimization for $W_{t}$ in the relaxed problem \eqref{eq:scheme_online NMF_surrogate}.

It seems that this might still require storing the entire history $X_{1},X_{2},\dots,X_{t}$ of data matrices up to time $t$. But in fact we only need to store two summary matrices $A_{t}\in \R^{r\times r}$ and $B_{t}\in \R^{r\times d}$. Indeed, \eqref{eq:scheme_online NMF_surrogate} is equivalent to the  optimization problem \eqref{eq:scheme_online NMF_surrogate2} stated in the introduction. To see this, note that 
\edit{\begin{align}
	\lVert X - WH\rVert_{F}^{2} &= \tr\left( (X-WH)(X-WH)^{T} \right) \\
	&=  \tr(W H H^{T} W^{T})  - 2\tr(W H X^{T}) + \tr(X X^{T}).
	\label{eq:trace1}
	\end{align} 
}
Hence if we let $A_{t}$ and $B_{t}$ be recursively defined as in \eqref{eq:scheme_online NMF_surrogate2}, then we can write
\begin{align}\label{eq:f_hat_t_trace}
\hat{f}_{t}(W) = \tr(W A_{t} W^{T})  - 2\tr(W B_{t}) + r_{t},
\end{align}
where $r_{t}$ does not depend on $W$. 
This explains the quadratic objective function for $W_{t}$ in \eqref{eq:scheme_online NMF_surrogate2}.

\vspace{0.2cm}
\section{Statement of main results}

\subsection{Setup and assumptions}\label{subsection:assumptions}

Fix integers $d,n,r\ge 1$ and a constant $\lambda>0$. Here we list all technical assumptions required for our convergence results to hold. 

\begin{customassumption}{(A1)}\label{assumption:A1}
	The observed data matrices $X_{t}$ are given by $X_{t}=\varphi(Y_{t})$, where $Y_{t}$ are drawn from a countable sample space $\Omega$ (\edit{hence $\Omega$ is measurable with respect to the counting measure}), and $\varphi:\Omega\rightarrow \mathbb{R}^{d\times n}$ is a bounded function. 
\end{customassumption}
\vspace{-0.3cm}
\begin{customassumption}{(A2)}\label{assumption:A2}
	Dictionaries $W_{t}$ are constrained to the subset $\mathcal{C}\subseteq \mathbb{R}^{d\times r}$, which is the disjoint union of compact and convex sets $\mathcal{C}_{1},\mathcal{C}_{2},\dots, \mathcal{C}_{m}$ in $\mathbb{R}^{d\times r}$.  
\end{customassumption}

\vspace{-0.3cm}
\begin{customassumption}{(M1)}\label{assumption:M1}
	The sequence of information $(Y_{t})_{t\in \mathbb{N}}$ is an irreducible, aperiodic, and positive recurrent Markov with state space $\Omega$.  We let $P$ and $\pi$ denote the transition matrix and unique stationary distribution of the chain $(Y_{t})_{t\in \mathbb{N}}$, respectively. 
\end{customassumption}
\vspace{-0.3cm}
\begin{customassumption}{(M2)}\label{assumption:M2}
	There exists a sequence $(a_{t})_{t\in \mathbb{N}}$ of \edit{non-decreasing integers such that }
	\begin{align}
	0\le a_{t}<t,\quad \sum_{t=1}^{\infty} w_{t-a_{t}}^{2}\sqrt{t}<\infty, \quad \sum_{t=1}^{\infty} w_{t}^{2}a_{t}<\infty,\quad  \sum_{t=1}^{\infty} w_{t} \sup_{\mathbf{y}\in \Omega} \lVert P^{a_{t}+1}(\mathbf{y},\cdot) - \pi \rVert_{TV} <\infty.
	\end{align}
\end{customassumption}
\vspace{-0.3cm}
\begin{customassumption}{(C1)}\label{assumption:C1}
	The loss and expected loss functions $\ell$ and $f$ defined in \eqref{eq:loss_def} and \edit{\eqref{eq:def_f}} are continuously differentiable and have Lipschitz gradient.  
\end{customassumption}
\vspace{-0.3cm}
\begin{customassumption}{(C2)}\label{assumption:C2}
	The eigenvalues of the positive semidefinite matrix $A_{t}$ defined in \eqref{eq:scheme_online NMF_surrogate2} are at least some constant $\kappa_{1}>0$ for \edit{all sufficiently large $t\in \mathbb{N}$.}
\end{customassumption}

It is standard to assume compact support for data matrices as well as dictionaries, which we do for as well in \ref{assumption:A1} and \ref{assumption:A2}. We remark that our analysis and main results still hold in the general state space case, but this requires a more technical notion of the positive Harris chains irreducibility assumption in order to use the functional central limit theorem for general state space Markov chains \citep[Thm. 17.4.4]{meyn2012markov}. We restrict our attention to the countable state space Markov chains in this paper.

The motivation behind assumptions~\ref{assumption:A1} and \ref{assumption:M1} is the following. If the sample space $\Omega$ as well as the desired distribution $\pi$ are complicated, then one may use a Markov chain Monte Carlo (MCMC) algorithm to sample information according to $\pi$, which will then be processed to form a meaningful data matrix. For the MCMC sampling, one designs a Markov chain on $\Omega$ that has $\pi$ as a stationary distribution, and then show that the chain is irreducible, aperiodic, and positive recurrent. Then by the general Markov chain theory we have summarized in Subsection \ref{subsection:MC_intro}, $\pi$ is the unique stationary distribution of the chain.

On the other hand, \ref{assumption:M2} is a very weak assumption on the rate of convergence of the Markov chain $(Y_{t})_{t\in \mathbb{N}}$ to its stationary distribution $\pi$. Note that \ref{assumption:M2} follows from 
\begin{customassumption}{(M2)'}\label{assumption:M2'}
	There exist constants $\beta\in (3/4, 1]$ and $\gamma>2(1-\beta)$ such that
	\begin{align}
	w_{t} = O(t^{-\beta}),\qquad \sup_{\mathbf{y}\in \Omega} \lVert P^{t}(\mathbf{y},\cdot) - \pi \rVert_{TV} = O(t^{-\gamma}).
	\end{align}
\end{customassumption}
Indeed, it is easy to verify that \ref{assumption:M2'} with $a_{t} = \lfloor \sqrt{t} \rfloor$ implies \ref{assumption:M2}. Furthermore, the mixing condition in \ref{assumption:M2'} is automatically satisfied when $\Omega$ is finite, which in fact covers many practical situations. Indeed, assuming \ref{assumption:M1} and that $\Omega$ is finite, the convergence theorem \eqref{eq:finite_MC_convergence_thm} provides an exponential rate of convergence of the empirical distribution of the chain to $\pi$, in particular implying the polynomial rate of convergence in \ref{assumption:M2'}.

Next, we comment on \ref{assumption:A2}. Our analysis holds as long as we can solve the quadratic minimization problem for $W_{t}$ under the constraint $\mathcal{C}$ intersected with the ellipsoid $\mathcal{E}_{t}:=\tr( (B_{t}^{T} - WA_{t} ) (W_{t-1}-W)^{T}) \le 0$ with Hessian matrix $A_{t}$ (see \eqref{eq:scheme_online NMF_surrogate2}). A particular instance of interest is when $\mathcal{C}$ is the disjoint union of convex constraint sets $\mathcal{C}_{i}$ as in \ref{assumption:A2}. By definition $A_{t}$ is positive semidefinite, so each $\mathcal{C}_{i}\cap \mathcal{E}_{t}$ is convex. Hence under \ref{assumption:A2}, we can solve $W_{t}$ in \eqref{eq:scheme_online NMF_surrogate2} by solving the convex sub-problems on each $\mathcal{C}_{i}\cap \mathcal{E}_{t}$ (see Algorithm \ref{algorithm:dictionary_update} for details, and also Figure \ref{fig:ellipsoid} right). This setting will be particularly useful for dictionary learning for multi-cluster data set, where it is desirable to find a dictionary that lies in one of multiple convex hulls of representative elements in each cluster \citep{peng2019online}. In the special case when $\mathcal{C}$ is convex, the additional ellipsoidal condition becomes redundant (see Proposition \ref{prop:gW_bd} (iii) and also Figure \ref{fig:ellipsoid} left), so the algorithm \eqref{eq:scheme_online NMF_surrogate2} as well as the assumption \ref{assumption:A2} reduce to the standard ones in  \citep{mairal2010online, mairal2013optimization, mairal2013stochastic}).  

\edit{Below we give a brief discussion about why we need to use the additional ellipsoidal constraint for the dictionary update in the general non-convex $\mathcal{C}$ case. A crucial ingredient in the convergence analysis of the OMF algorithm (Algorithm \ref{eq:scheme_online NMF_surrogate2}) is the following so-called `second-order growth property':
	\begin{align}\label{eq:second_order_growth_original}
	g_{t}(W_{t}) - g_{t}(W_{t-1}) \ge c \lVert W_{t}-W_{t-1} \rVert_{F}^{2} \quad \text{for all $t\ge 1$}.
	\end{align}
	Here $g_{t}$ is the quadratic function defined by $g_{t}(W)=\tr(W A_{t} W^{T})  - 2\tr(W B_{t})$, where $A_{t}$ and $B_{t}$ are recursively computed by Algorithm \ref{eq:scheme_online NMF_surrogate2}. Roughly speaking, it should be guaranteed that we improve in minimizing $g_{t}$ quadradically in the change of the dictionary matrix $W_{t}-W_{t-1}$. However, when $\mathcal{C}$ is non-convex, it is possible that one may pay a large change from $W_{t-1}$ to $W_{t}$ and still do not gain any meaningful improvement in minimizing $g_{t}$ (see Figure \ref{fig:ellipsoid}, middle). The ellipsoidal constraint $\mathcal{E}_{t}$ ensures the second-order growth property \eqref{eq:second_order_growth_original} when we optimize $g_{t}$ in $\mathcal{C}\cap \mathcal{E}_{t}$ (see Proposition \ref{prop:gW_bd} (i)). }

\begin{figure*}[h]
	\centering
	\includegraphics[width=1 \linewidth]{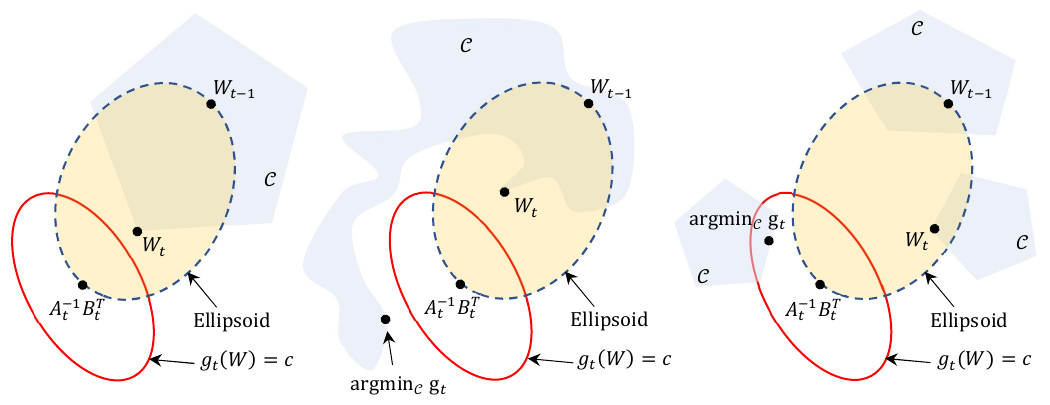}
	\vspace{-0.2cm}
	\caption{
		Illustration of the constraint quadratic minimization problem for the dictionary update step in the main algorithm \eqref{eq:scheme_online NMF_surrogate2}. There, the updated dictionary matrix $W_{t}$ is found by minimizing the quadratic function $g_{t}(W)=\tr(W A_{t} W^{T})  - 2\tr(W B_{t})$ (global minimum at $A_{t}^{-1}B_{t}^{T}$), where $A_{t}$ and $B_{t}$ are recursively computed by the algorithm \eqref{eq:scheme_online NMF_surrogate2}, over the constraint set $\mathcal{C}$ intersected with the ellipsoid $\mathcal{E}_{t}:=\tr( (B_{t}^{T} - WA_{t} ) (W_{t-1}-W)^{T}) \le 0$. (Left) If $\mathcal{C}$ is convex, the additional ellipsoidal constraint becomes redundant. (Middle) If $\mathcal{C}$ is non-convex, the second-order growth condition \eqref{eq:second_order_growth_original} may be violated by the minimizer of $g_{t}$ in $\mathcal{C}$ but is still guaranteed with the ellipsoidal condition. (Right) When $\mathcal{C}$ is the disjoint union of convex sets, then the constraint quadratic problem on $\mathcal{C}\cap \mathcal{E}_{t}$ can be exactly solved.  
	}
	\label{fig:ellipsoid}
\end{figure*}

Our main result, Theorem \ref{thm:online NMF_convergence}, guarantees that both the empirical and surrogate loss processes $(f_{t}(W_{t}))$ and $(\hat{f}_{t}(W_{t}))$ converge almost surely under the assumptions \ref{assumption:A1}-\ref{assumption:A2} and \ref{assumption:M1}-\ref{assumption:M2}\edit{.} The assumptions \ref{assumption:C1}-\ref{assumption:C2}, which are also used in \citep{mairal2010online}, are sufficient to ensure that the limit point is a critical point of the expected loss function $f$ defined in \eqref{eq:def_f}.

We remark that \ref{assumption:C1} follows from the following alternative condition (see \citep[Prop. 1]{mairal2010online}):
\begin{customassumption}{(C1)'}\label{assumption:C1'}
	For each $X\in \Omega$ and $W\in \mathcal{C}$, the sparse coding problem in \eqref{eq:loss_def} has a unique solution.
\end{customassumption}
In order to enforce \ref{assumption:C1'}, we may use the elastic net penalization by Zou and Hastie \citep{zou2005regularization}. Namely, we may replace the first equation in \eqref{eq:scheme_online NMF_surrogate2} by 
\begin{align}\label{eq:Ht_modified}
H_{t} = \argmin_{H\in \mathcal{C}'\subseteq \R^{r\times n}} \lVert X_{t} - W_{t-1}H \rVert_{F}^{2} + \lambda \lVert H \rVert_{1} + \frac{\kappa_{2}}{2} \lVert H \rVert_{F}^{2}
\end{align}
for some fixed constant $\kappa_{2}>0$. See the discussion in \citep[Subsection 4.1]{mairal2010online} for more details.

On the other hand, \ref{assumption:C2} guarantees that the eigenvalues of $A_{t}$ produced by \eqref{eq:scheme_online NMF_surrogate2} are lower bounded by the constant $\kappa_{1}>0$. It follows that $A_{t}$ is invertible and $\hat{f}_{t}$ is strictly convex with Hessian $2A_{t}$. This is crucial in deriving Proposition \ref{prop:Wt_increment_bd}, which is later used in the proof of Theorem \ref{thm:online NMF_convergence}. Note that \ref{assumption:C2} can be enforced by replacing the last equation in \eqref{eq:scheme_online NMF_surrogate2} with  
\begin{align}\label{eq:Wt_modified}
W_{t}& = \argmin_{W\in \mathcal{C} \subseteq \R^{d\times r}} \left(  \tr\left(W (A_{t}+\kappa_{1}I) W^{T}\right)  - 2\tr(W B_{t})\right)  \\
&\hspace{0.85cm} \text{s.t.} \,\, \tr( (B_{t}^{T} - WA_{t} ) (W_{t-1}-W)^{T}) \le 0. 
\end{align}
The same analysis for the algorithm \eqref{eq:scheme_online NMF_surrogate2} that we will develop in the later sections will apply for the modified version with \eqref{eq:Ht_modified} and \eqref{eq:Wt_modified}, for which \ref{assumption:C1}-\ref{assumption:C2} are trivially satisfied.

\subsection{Statement of main results}

Our main result in this paper, which is stated below in Theorem \ref{thm:online NMF_convergence}, asserts that under the OMF scheme \eqref{eq:scheme_online NMF_surrogate2}, the induced stochastic processes $(f_{t}(W_{t}))_{ t\in \mathbb{N}}$ and $(\hat{f}_{t}(W_{t}))_{ t\in \mathbb{N}}$ converge as $t\rightarrow \infty$ in expectation. Furthermore, the sequence $(W_{t})_{ t\in \mathbb{N}}$ of learned dictionaries converge to the set of critical points of the expected loss function $f$.

\begin{theorem}\label{thm:online NMF_convergence}
	Suppose \ref{assumption:A1}-\ref{assumption:A2} and \ref{assumption:M1}-\ref{assumption:M2}. Let $(W_{t}, H_{t})_{t\ge 1}$ be a solution to the optimization problem \eqref{eq:scheme_online NMF_surrogate2}. Further assume that $\sum_{t=1}^{\infty} w_{t}=\infty$. Then the following hold. 
	\begin{description}
		\item[(i)] $\lim_{t\rightarrow \infty} \E[f_{t}(W_{t})]  =\lim_{t\rightarrow \infty} \E[\hat{f}_{t}(W_{t})] <\infty$. 
		
		\vspace{0.1cm}	
		\item[(ii)] $f_{t}(W_{t})-\hat{f}_{t}(W_{t})\rightarrow 0$ and $f(W_{t})-\hat{f}_{t}(W_{t})\rightarrow 0$ as $t\rightarrow \infty$ almost surely.

		\vspace{0.1cm}	
		\item[(iii)] Further assume \ref{assumption:C1}-\ref{assumption:C2}. Then almost surely, 
		\begin{align} 
		\lim_{t\rightarrow\infty} \lVert \nabla_{W} f (W_{t}) - 2(W_{t}A_{t}-B_{t}) \rVert_{F} = 0.
		\end{align}
		Furthermore, the distance between $W_{t}$ and the set of all local extrema of $f$ in $\mathcal{C}$ converges to zero almost surely.  
	\end{description}
\end{theorem}

The second part of Theorem \ref{thm:online NMF_convergence} (ii) is a new result based on the first part and a uniform convergence result in Lemma \ref{lem:uniform_convergence_asymmetric_weights}. This implies that for large $t$, the true objective function $f(W_{t})$, which requires averaging over random data matrices sampled from the stationary distribution $\pi$, can be approximated by the easily computed surrogate objective $\hat{f}_{t}(W_{t})$. If we further assume that the global minimizer of the quadratic function $g(W)=\tr(WA_{t}W^{T}) - 2\tr(WB_{t})$ is in the interior of the constraint set $\mathcal{C}$, then $\nabla_{W}g(W_{t}) = 2(W_{t}A_{t} - B_{t}) \equiv 0$, so Theorem \ref{thm:online NMF_convergence} (iii) yields $\lVert \nabla_{W} f(W_{t}) \rVert_{F}\rightarrow 0$ almost surely as $t\rightarrow \infty$. 
We also remark that in the special case of convex constraint set $\mathcal{C}$ for dictionaries, i.i.d. data matrices $(X_{t})_{t\in \mathbb{N}}$, and balanced weights $w_{t}=1/t$, our results above recover the classical results in \citep[Prop. 2 and 3]{mairal2010online}. For general weighting scheme and objective functions, similar results were obtained \citep{mairal2013stochastic} using similar proof techniques.

As discussed in Subsection \ref{subsection:contribution}, the core of our proof of Theorem \ref{thm:online NMF_convergence} is to use conditioning on distant past in order to allow the Markov chain to mix close enough to the stationary distribution $\pi$. This allows us to control the difference between the new and the average losses by concentration of Markov chains (see Proposition \ref{prop:increment_bd} and Lemma \ref{lemma:increment_bd}), overcoming the limitation of the quasi-martingale based approach typically used for i.i.d. input \citep{mairal2010online, mairal2013optimization, mairal2013stochastic}. A practical implication of the theoretical result is that one can now extract features more efficiently from the same dependent data streams, as there is no need to subsample the data sequence to approximately satisfy the independence assumption. 

It is worth comparing our approach that directly handles Markovian dependence to the popular approach using subsampling. Namely, if we keep only one Markov chain sample in every $\tau$ iterations (subsampling epoch) then the remaining samples are asymptotically independent provided the epoch $\tau$ is long enough compared to the mixing time of the Markov chain. A similar line of approach was used in \citep{yang2019online} for a relevant but different problem of factorizing the unknown transition matrix of a Markov chain by observing its trajectory. However, there are a number of shortcomings in the approach based on subsampling. First, consecutive samples obtained after subsampling are nearly independent, but never completely independent. Hence subsampling dependent sequences does not rigorously verify independence assumption. Second, subsampling makes use of only a small portion of the already obtained samples, which may be not be of the most efficient use of given data samples. Our approach directly handles Markovian dependence in data samples and hence does not suffer from these shortcomings. See Section \ref{section:Ising} for a numerical verification of this claim.   

\vspace{0.2cm}

\section{Application I: Learning features from MCMC trajectories}
\label{section:Ising}

In this section, we demonstrate learning features from a single MCMC trajectory of dependent samples in the case of two-dimensional \textit{Ising model}, which was first introduced by Lenz in 1920 for modeling ferromagnetism \citep{lenz1920beitrvsge} and has been one of the most well-known spin systems in the physics literature. Our analysis for the Ising model could easily be generalized to the other well-known spin systems such as the Potts model \citep{wu1982potts}, the cellular Potts model \citep{ouchi2003improving, maree2007cellular, szabo2013cellular}, or the restricted Boltzmann Machine \citep{nair2010rectified}.



\subsection{Spin systems and the Ising model}
\label{subsection:Ising}

Consider a general system of binary spins. Namely, let $G=(V,E)$ be a simple graph with vertex set $V$ and edge set $E$. Imagine each vertex (site) of $G$ can take either of the two states (spins) $+1$ or $-1$. A complete assignment of spins for each site in $G$ is given by a \textit{spin configuration}, which we denote by a map $\mathbf{x}:V\rightarrow \{-1,1\}$. Let \edit{$\Omega=\{-1,1\}^{V}$} denote the set of all spin configurations. In order to introduce a probability measure on $\Omega$, fix a function \edit{$H(\cdot;\theta):\Omega\rightarrow \mathbb{R}$} parameterized by $\theta$, which is called a \textit{Hamiltonian} of the system. For each choice of parameter $\theta$, we define a probability distribution $\pi_{\theta}$ on the set $\Omega$ of all spin configurations by 
\begin{align}\label{eq:def_gibbs_measure}
\pi_{\theta}(\x) = \frac{1}{Z_{\theta}} \exp\left(  -H(\x;\theta) \right),
\end{align}
where the partition function $Z_{\theta}$ is defined by $Z_{\theta}  = \sum_{\x\in \Omega} \exp\left( -H(\x;\theta)\right)$. The induced probability measure $\P_{\theta}$ on $\{-1,1\}^{V}$ is called a \textit{Gibbs measure}. 

The Ising model is defined by the following Hamiltonian 
\begin{align}
H(\mathbf{x};T,h) = \frac{1}{T}\left(-\sum_{\{u,v\}\in E} \mathbf{x}(u)\mathbf{x}(v) - \sum_{v\in V} h(v)\mathbf{x}(v) \right),
\end{align}
where $\x$ is the spin configuration, the parameter $T$ is called the \textit{temperature}, and $h:V\rightarrow \mathbb{R}$ the \textit{external field}. In this paper we will only consider the case of zero external field. Note that, with respect to the corresponding Gibbs measure, a given spin configuration $\mathbf{x}$ has higher probability if the adjacent spins tend to agree, and this effect of adjacent spin agreement is emphasized (resp., diminished) for low (resp., high) temperature $T$. Different choice of the spin space $\Omega$ and the Hamiltonian $H$ lead to other well-known spin systems such as the Potts model, cellular Potts Model, and the restricted Bolzman Machine \citep{nair2010rectified} (see the references given before).

\subsection{Gibbs sampler for the Ising model}
\label{subsection:Gibbs_ising}
One of the most extensively studied Ising models is when the underlying graph $G$ is the two-dimensional square lattice (see
\citep{mccoy2014two} for a survey). It is well known that in this case the Ising model exhibits a sharp phase transition at the critical temperature $T=T_{c}=2/\log(1+\sqrt{2}) \approx 2.2691$. Namely, if $T<T_{c}$ (subcritical phase), then there tends to be large clusters of $+1$'s and $-1$ spins; if $T>T_{c}$ (supercritical phase), then the spin clusters are very small and fragmented; at $T=T_{c}$ (criticality), the cluster sizes are distributed as a power law. 

In order to sample a random spin configuration $\mathbf{x}\in \Omega$, we use the following MCMC called the \textit{Gibbs sampler}. Namely, let the underlying graph $G=(V,E)$ \edit{be} a finite $N\times N$ square lattice. We evolve a given spin configuration $\mathbf{x}_{t}:V\rightarrow \{-1,1\}$ at iteration $t$ as follows:
\begin{description}
	\item[(i)] Choose a site $v\in V$ uniformly at random;
	\item[(ii)] Let $\x^{+}$ and $\x^{-}$ be the spin configurations obtained from $\x_{t}$ by setting the spin of $v$ to be $1$ and $-1$, respectively. Then  
	\begin{align}
	p(\x_{t},\, \x^{+}) = \frac{\pi(\x^{+})}{\pi(\x^{+}) + \pi(\x^{-})}, \qquad 	p(\x_{t},\, \x^{-}) = \frac{\pi(\x^{-})}{\pi(\x^{+}) + \pi(\x^{-})}.
	\end{align}
\end{description} 
Note that $p(\x_{t+1},\, \x^{+}) = (1+\exp(2T^{-1}\sum_{u\sim v} \mathbf{x}_{t}(u)))^{-1}$, where the sum in the exponential is over all neighbors $u$ of $v$. Iterating the above transition rule generates a Markov chain trajectory $(\mathbf{x}_{t})_{ t\in \mathbb{N}}$ of Ising spin configurations, and it is well known that it is irreducible, aperiodic, and has the Boltzmann distribution $\pi_{T}$ (defined in \eqref{eq:def_gibbs_measure}) as its unique stationary distribution. 


\subsection{Learning features from Ising spin configurations}

Suppose we want to learn features from a random element $X$ in a sample space $\Omega$ with distribution $\pi$. When the sample space is complicated, it is often not easy to directly sample a random element from it according to the prescribed distribution $\pi$. While Markov chain Monte Carlo (MCMC) provides a fundamental sampling technique that uses Markov chains to sample a random element (see, e.g., \citep{levin2017markov}), it inherently generates a dependent sequence of samples. In order to satisfy the independence assumption in most online learning algorithms, either one may generate each sample by a separate MCMC trajectory, or uses subsampling to reduce the dependence in a given MCMC trajectory. Both of these approaches suffer from inefficient use of already obtained data sample and never guarantee perfect independence. However, as our main result (Theorem \ref{thm:online NMF_convergence}) guarantees almost sure convergence of the dictionaries under Markov dependence, we may use arbitrary (or none) subsampling epoch to optimize the learning for a given MCMC trajectory. We will show that learning features from dependent sequence of data yields qualitatively different outcome, and also significantly improves efficiency of the learning process.

\begin{figure*}[h]
	\centering
	\includegraphics[width=1 \linewidth]{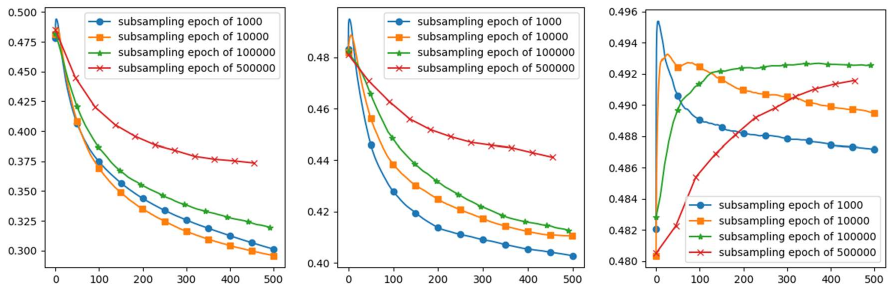}
	\vspace{-0.5cm}
	\caption{
		Plot of surrogate losses (normalized by $4\times 10^{4}$) vs. MCMC iterations (unit $10^4$) for subsampling epochs of 1000, 10000, 100000, and 500000 for temperatures $T=0.5$ (left), $T=2.26$ (middle), and $T=5$ (right), respectively. 
	}
	\label{fig:ising_error_plot}
\end{figure*}

\begin{figure*}[h]
	\centering
	\includegraphics[width=1 \linewidth]{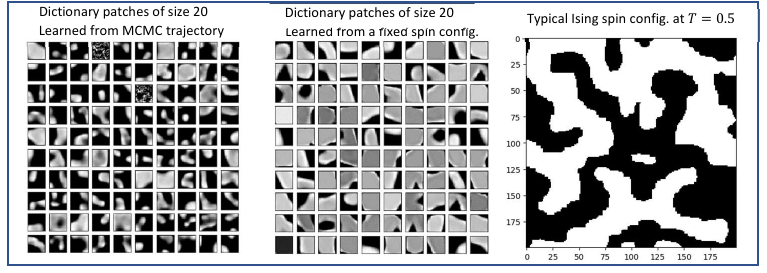}
	\vspace{-0.5cm}
	\caption{
		(Left) 100 learned dictionary patches from a MCMC Gibbs sampler for the Ising model on $200\times 200$ square lattice at a subcritical temperature ($T=0.5$). (Middle) 100 learned dictionary patches from fixed Ising spin configuration at $T=0.5$ shown in the right.   
	}
	\label{fig:ising_NMF_sub}
\end{figure*}

We first describe the setting of our simulation of online NMF algorithm on the Ising model. We consider a Gibbs sampler trajectory $(\mathbf{x}_{t})_{ t\in \mathbb{N}}$ of the Ising spin configurations on the $200\times 200$ square lattice at three temperatures $T=0.5$, $2.26$, and $5$. Initially $\mathbf{x}_{0}$ is sampled so that each site takes $+1$ or $-1$ spins independently with equal probability, and we run the Gibbs sampler for $5\cdot 10^{6}$ iterations. We use four different subsampling epochs $\tau=1000,10000,100000,$ and $500000$ for online dictionary learning. Namely, every $\tau$ iterations, we obtain a coarsened MCMC trajectory, which is represented as a $200\times 200 \times (5\cdot 10^{6}/\tau)$ array $A$ whose $k$th array $A[:,\,:,\,k]$ corresponds to the spin configuration $X_{k}:=\mathbf{x}_{\tau k}$. Then $(X_{k})_{k\ge 0}$ also defines an irreducible and aperiodic Markov chain on $\Omega$ with the same stationary distribution $\pi_{T}$. For each $X_{k}$, which is a $200\times 200$ matrix of entries from $\{-1,1\}$, we sample 1000 patches of size $20\times 20$. After flattening each patch into a column vector, we obtain a $400\times 1000$ matrix, which we denote by $\texttt{Patch}_{20}(X_{k})$. We apply the online NMF scheme to the Markovian sequence $(\texttt{Patch}_{20}(X_{k}))_{k\ge 0}$ of data matrices to extract 100 dictionary patches. 

We remark even with the largest subsampling epoch $\tau=500000$ of our choice, consecutive spin configurations are far from being independent, especially at low temperature $T=0.5$. Notice that by the coupon collector's problem, with high probability, we need $40000\log 40000\approx 423865$ iterations so that each of the $40000$ nodes in the lattice gets resampled at least once. As changes only  occur at the interface between the $+1$ and $-1$ spins, the interfaces will barely move during this epoch, especially at the low-temperature case $T=0.5$. At one extreme, two configurations that are $\tau=1000$ iterations at $T=0.5$ apart will look almost identical, hence with significant correlation. However, in all cases, as the chain $(X_{t})_{\edit{t\in \mathbb{N}}}$ is irreducible, aperiodic, and on a finite state space, we can apply the main theorem (Theorem \ref{thm:online NMF_convergence}) to guarantee the almost sure convergence of the dictionary patches to the set of critical points of the expected loss function \eqref{eq:def_f}.

In Figure \ref{fig:ising_error_plot}, we plot the surrogate loss $\hat{f}_{t}(W_{t})$ in all 12 cases of different temperatures and subsampling epochs. By Theorem \ref{thm:online NMF_convergence} (ii), the surrogate loss is a close approximation of the true objective (the expected loss) for large $t$, which should normally be computed using a separate Monte Carlo integration. Notice that in all cases, the Gibbs sampler is run for the same $5\times 10^{7}$ iterations. Since we do not have to worry about dependence in the samples for our convergence theorem to hold, one might expect to get steeper decrease in the surrogate loss at the shorter subsampling epoch, as it enables training dictionaries over more samples. We indeed observe such results in Figure \ref{fig:ising_error_plot}. Interestingly, for $T=0.5$, we observe that longer subsampling epoch $\tau=10000$ gives faster decay in the surrogate error than $\tau=1000$ does. This can be explained as an overfitting issue. Namely, since two configurations that are 1000 iterations apart are barely different, training dictionaries too frequently may overfit them to specific configurations, while the objective is to learn from the average configuration. The dictionaries learned from each of the 12 simulations are shown in Figure \ref{fig:ising_dicts_subsampling}.

In Figures \ref{fig:ising_NMF_sub}, \ref{fig:ising_NMF_cri}, and \ref{fig:ising_NMF_sup}, we compare 100 learned dictionary elements directly from the MCMC trajectory $(X_{k})_{0\le k \le 500}$ with subsampling epoch $\tau=10000$ as well as from a fixed spin configuration for the three temperatures $T=0.5$ (subcritical), $T=2.26$ (near critical), and $T=5$ (supercritical). In all three figures, we see qualitative differences between the two sets of dictionary elements, which can be explained as follows. Since spin configurations gradually change in the MCMC trajectory $(X_{k})_{0\le k \le 500}$, the dictionary elements learned from the trajectory should not be overfitted to a particular configuration as the one in the middle of each figure does, but should capture features common to a number of spin configurations at the corresponding temperature.

\vspace{0.3cm}

\section{Application II: Network dictionary learning by online NMF and motif sampling}
\label{section:NDL}

In this section, we propose a novel framework for network data analysis that we call \textit{Network Dictionary Learning}, which enables one to extract ``network dictionary patches'' from a given network to see its fundamental features and to reconstruct the network using the learned dictionaries. Network Dictionary Learning is based on two building blocks: 1) Online NMF on Markovian data, which is the main subject in this paper, and 2) a recent MCMC algorithm for sampling motifs from networks in  \citep{lyu2019sampling}. More details on Network Dictionary Learning and applications to social networks will be given in an upcoming paper \citep{lyu2020dictionary}.

\subsection{Extracting patches from a network by motif sampling}
\label{ref:NDL_simple}


\edit{We formally define a \textit{network} as a pair $\G=(V,A)$ of node set $V$ and a \textit{weight matrix} $A:V^{2}\rightarrow [0,\infty)$ describing interaction strengths between the nodes. In this formulation, we do not distinguish between multi-edges and weighted edges in networks. A given graph $G=(V,E)$ determines a unique network $\G=(V, A_{G})$ with $A_{G}$ the adjacency matrix of $\G$. We call a network $\G=(V,A)$ \textit{simple} if $A$ is symmetric, binary (i.e., $A(x,y)\in \{0,1\}$), and without self-edges (i.e., $A(x,x)=0$). We identify a simple graph $G=(V,E)$ with adjacency matrix $A_{G}$ with the simple network $\G=(V,A_{G})$. }

For networks, we can think of a $(k\times k)$ patch as a sub-network induced onto a subset of $k$ nodes. As we imposed to select $k$ consecutive rows and columns to get patches from images, we need to impose a reasonable condition on the subset of nodes so that the selected nodes are strongly associated. For instance, if the given network is sparse, selecting three nodes uniformly at random would rarely induce any meaningful sub-network. Selecting such a subset of $k$ nodes from networks can be addressed by the \textit{motif sampling} technique introduced in \citep{lyu2019sampling}. Namely, for a fixed ``template graph'' (motif) $F$ of $k$ nodes, we would like to sample $k$ nodes from a given network $\G$ so that the induced sub-network always contains a copy of $F$. This guarantees that we are always sampling some meaningful portion of the network, where the prescribed graph $F$ serves as a backbone. \edit{More precisely, fix an integer $k\ge 1$ and a matrix $A_{F}:[k]^{2}\rightarrow [0,\infty)$, where $[k]=\{1,2,\dots,k\}$. The corresponding network $F=([k],A_{F})$ is called a \textit{motif}. The particular motif of our interest is the \textit{$k$-chain}, where $A_{F}=\mathbf{1}(\{(1,2),(2,3),\ldots,(k-1,k) \})$. The $k$-chain motif corresponds to a directed path with node set $[k]$.}

Based on these ideas, we propose the following preliminary version of Network Dictionary Learning for simple graphs. 

\begin{description}
	\item{\textbf{Network Dictionary Learning (NDL): Static version for simple graphs}}
	\vspace{0.1cm}
	\item[(i)] Given a simple graph $\G=(V,A)$ and a motif $F=([k],A_{F})$, let $\mathtt{Hom}(F,\G)$ denote the set of all \textit{homomorphisms} $F\rightarrow \G$:
	\begin{align}
	\mathtt{Hom}(F,\G) = \left\{ \mathbf{x}:[k]\rightarrow [n]\,\bigg|\, \prod_{1\le i,j\le k} A(\mathbf{x}(i), \mathbf{x}(j))^{A_{F}(i,j)} =1 \right\}.
	\end{align}
	Compute $\mathtt{Hom}(F,\G)$ and write $\mathtt{Hom}(F,\G)=\{\x_{1},\x_{2},\dots,\x_{N}\}$.
	
	\vspace{0.1cm}
	\item[(ii)] For each homomorphism $\mathbf{x}:F\rightarrow \G$, associate a $(k\times k)$ matrix $A_{\mathbf{x}}$ by 
	\begin{align}\label{eq:def_F_patch}
	A_{\mathbf{x}}(a,b) = A(\mathbf{x}(a), \mathbf{x}(b)) \qquad 1\le a,b\le k.
	\end{align}
	Let $X$ denote the $(k^{2}\times N)$ matrix whose $i$th column is the $k^{2}$-dimensional vector obtained by vectorizing $A_{\mathbf{x}_{i}}$ (using the lexicographic ordering)\edit{.}
	
	\vspace{0.1cm}
	\item[(iii)] Factorize $X \approx WH$ using NMF. Reshaping the columns of the dictionary matrix $W$ into $k\times k$-squares \edit{gives} the learned network dictionary elements. 
\end{description}

\subsection{Motif sampling from networks and MCMC sampler}

There are two main issues in the preliminary Network Dictionary Learning scheme for simple graphs we described in the previous subsection. First, computing the full set $\mathtt{Hom}(F,\G)$ of homomorphisms $F\rightarrow \G$\footnote{When $\G$ is a complete graph $K_{q}$ with $q$ nodes, computing homomorphisms $F\rightarrow K_{q}$ equals to computing all proper $q$-colorings of $F$.} is computationally expensive with $O(n^{k})$ complexity. Second, in the case of the general network with edge and node weights, some homomophisms could be more important in capturing features of the network than others. In order to overcome the second difficulty, we introduce a probability distribution for the homomorphisms for the general case that takes into  account the weight information of the network. To handle the first issue, we use a MCMC algorithm to sample from such a probability measure and apply online NMF to sequentially learn network dictionary patches.

For a given motif $F=([k],A_{F})$ and a $n$-node network $\G=(V,A)$, we introduce the following probability distribution $\pi_{F\rightarrow \G}$ on the set $V^{[k]}$ of all vertex maps $\mathbf{x}:[k]\rightarrow V$ by 
\begin{equation}\label{eq:def_embedding_F_N}
\pi_{F\rightarrow \G}( \mathbf{x} ) = \frac{1}{\mathtt{Z}}  \left( \prod_{1\le i,j\le k}  A(\mathbf{x}(i),\mathbf{x}(j))^{A_{F}(i,j)} \right),
\end{equation}  
where $\mathtt{Z}$ is the normalizing constant called the \textit{homomorphism density} of $F$ in $\G$. We call the random vertex map $\x:[k]\rightarrow V$ distributed as $\pi_{F\rightarrow \G}$ the \textit{random homomorphism} of $F$ into $\G$. Note that $\pi_{F\rightarrow \G}$ becomes the uniform distribution on the set of all homomorphisms $F\rightarrow \G$ when both $A$ and $A_{F}$ are binary matrices.

\begin{figure*}[h]
	\centering
	\includegraphics[width=0.8 \linewidth]{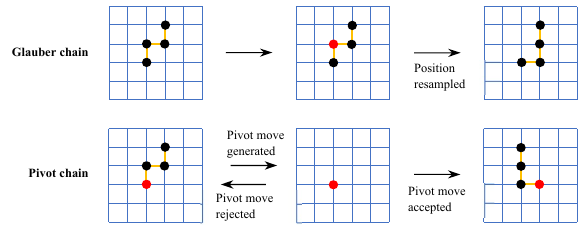}
	\caption{ Single iterations of the Glauber chain (first row) and Pivot chain (second row) sampling of homomorphisms $\x_{t}:F \rightarrow \G$, where  $\G$ is the $(6\times 6)$ grid and $F=([4],\mathbf{1}_{\{(1,2),(2,3),(3,4)\} })$ is the 4-chain motif. For the Glauber chain, a node $v\in [4]$ (in red) in the motif is sampled uniformly at random and its position $\x_{t}(v)$ is resampled to $\x_{t+1}(v)$. For the Pivot chain, a random walk move  $\x_{t}(1)\rightarrow \x_{t+1}(1)$ of the pivot (in red) is generated, which is accepted or rejected according to the acceptance probability computed by the Metropolis-Hastings algorithm. If it is accepted, positions of the subsequent nodes $\x_{t+1}(2),\x_{t+1}(3),\x_{t+1}(4)$ are sampled successively; otherwise, we take $\x_{t+1}\leftarrow \x_{t}$.   
	}
	\label{fig:Glauber}
\end{figure*}

In order to sample a random homomorphism $F\rightarrow \G$, we use the MCMC algorithms introduced in \citep{lyu2019sampling} called the \textit{Glauber chain} (see Algorithm \ref{alg:glauber}) and the \textit{Pivot chain} (see Algorithm \ref{alg:pivot}). The Glauber chain an  exact analogue of the Gibbs sampler for the Ising model we discussed in Subsection \ref{subsection:Gibbs_ising}. Namely, for the update $\x_{t}\mapsto \x_{t+1}$, we pick one node $v\in [k]$ of $F$ uniformly at random, and resample the time-$t$ position $\x_{t}(v)\in V$ of node $i$ in the network $\G$ from the correct conditional distribution. (See the first row of Figure \ref{fig:Glauber} for an illustration.) 

\edit{The Pivot chain is a combination of a random walk on networks and the Metropolis-Hastings algorithm. There, node 1 in the motif is designated as the `pivot'. For each update $\x_{t}\mapsto \x_{t+1}$, we first generate a random walk move $\x_{t}(1)\rightarrow \x_{t+1}(1)$ of the pivot (e.g., when the network is simple, then $\x_{t+1}(1)$ is a uniformly chosen neighbor of $\x_{t}(1)$), which is then accepted with a suitable acceptance probability (see \eqref{eq:pivot_chain_acceptance_prob}) according to the Metropolis-Hastings algorithm (see, e.g., \citep[Sec. 3.2]{levin2017markov}). If rejected, we take $\x_{t+1}\leftarrow \x_{t}$; otherwise, each $\x_{t+1}(i)\in V$ for $i=2,3,\dots,\kappa$ is sampled successively from the appropriate conditional distribution (see \eqref{eq:pivot_conditional}) so that the stationary distribution such that the resulting Markov chain is the desired distribution  $\pi_{F\rightarrow \G}$ in \eqref{eq:def_embedding_F_N} as its unique stationary distribution. (See the second row of Figure \ref{fig:Glauber} for an illustration.) }

\edit{In Algorithm \ref{alg:pivot}, we provide a variant of the original Pivot chain in \citep{lyu2019sampling} that uses an approximate computation of the correct acceptance probability \eqref{eq:pivot_chain_acceptance_prob} with the boolean variable $\texttt{AcceptProb}=\texttt{Approximate}$. This is to reduce the computational cost of computing the exact acceptance probability, which could be costly when the motif $F$ has large number of nodes. (See \citep{lyu2020dictionary} for a more detailed discussion.)}

\subsection{Algorithms for Network Dictionary Learning and Reconstruction}
\label{subsection:NDL_algs}

In Algorithm \ref{alg:NDL}, we give the algorithm for Network Dictionary Learning that combines online NMF and MCMC motif sampling with the ideas that we described in \eqref{ref:NDL_simple}. Below we give a high-level description of Algorithm \ref{alg:NDL}.

(\edit{The detailed algorithm is given in the appendix})

\begin{description}
	\item{\textbf{Network Dictionary Learning (NDL): Online version for general networks}}
	\vspace{0.1cm}
	\item{} Given a simple graph $\G=(V,A)$ and a motif $F=([k],A_{F})$, do the following for $t=1,2,\dots, T$: 
	\vspace{0.1cm}
	\item[(i)] Generate $N$ homomorphisms $\x_{s}$ for $N(t-1)\le s\le Nt$ using a MCMC motif sampling algorithm
	
	\vspace{0.1cm}
	\item[(ii)] Compute $N$ $(k\times k)$ matrices $A_{\x_{s}}$ by \eqref{eq:def_F_patch}. Let $X_{t}$ be the $(k^{2}\times N)$ matrix whose $j$th column is the vectorization of $A_{\x_{\ell}}$ with $\ell=N(t-1)+j$.  
	
	\vspace{0.1cm}
	\item[(iii)] Update the previous dictionary matrix $W_{t-1}\in \R_{\ge 0}^{k^{2}\times N}$ to $W_{t}$ with respect to the new data matrix $X_{t}$ using Online NMF.
\end{description}  

At each iteration $t=1,2,\dots,T$,	a chosen motif sampling algorithm generates a sequence of $N$ homomorphisms $\x_{s}:F\rightarrow \G$ and corresponding $(k\times k)$ matrices $A_{\x_{s}}$, which are summarized as the $(k^{2}\times N)$ data matrix $X_{t}$. The online NMF algorithm (Algorithm \ref{algorithm:online NMF}) then  learns a nonnegative factor matrix $W_{t}$ of size $(k^{2}\times r)$ by improving the previous factor matrix $W_{t-1}$ with respect to the new data matrix $X_{t}$. Note that during this entire process, the algorithm only needs to hold two auxiliary matrices $P_{t}$ and $Q_{t}$ of fixed sizes $(r\times r)$ and $(r\times k^{2})$, respectively, but not the previous data matrices $X_{1},\dots, X_{t-1}$. Hence NDL is efficient in memory and scalable in the network size. Moreover, NDL is applicable for temporally changing networks due to its online nature.

Next, in Algorithm \ref{alg:network_reconstruction}, we provide an algorithm that reconstructs a given network using a network dictionary learned by Algorithm \ref{alg:NDL}. The idea behind our network reconstruction algorithm is the following. Similarly as in image reconstruction, network reconstruction is done by first sampling random $k$-node subgraphs that contain the corresponding motifs, whose adjacency matrices are approximated by the learned dictionary atoms. Then these reconstructions are patched together with a suitable averaging. However, unlike image reconstruction where we can easily access any desired $(k\times k)$ square patch, for network reconstruction, we cannot directly sample a random $k$-node subgraph that contains a fixed motif. For this, we use the Markov chain $(\x_{t})_{ t\in \mathbb{N}}$ of homomorphisms $F\rightarrow \G$ as we do in network dictionary learning. For each $t\in \mathbb{N}$, we reconstruct the $k\times k$ patch of $\G$ corresponding to the current homomorphism $\x_{t}$ using the learned dictionaries. We keep track of the overlap count for each entry $A(a,b)$ that we have reconstructed up to time $t$, and take the average of all the proposed values of each entry $A(a,b)$ up to time $t$.

As a corollary of our main result (Theorem \ref{thm:online NMF_convergence}) and \citep[Thm 5.7]{lyu2019sampling} for the convergence of the Glauber and Pivot chains for homomorphisms, we obtain the following convergence guarantee of Algorithm \ref{alg:NDL} for Network Dictionary Learning. 

\begin{corollary}\label{cor:NDL}
	Let $F=([k],A_{F})$ be the $k$-chain motif and let  $\mathcal{G}=(V,A)$ be a network that satisfies the following properties:
	\begin{description}
		\item{\textup{(i)}} Random walk on $\G$ is irreducible and aperiodic;
		\item{\textup{(ii)}} $\G$ is bidirectional, that is, $A(a,b)>0$ implies $A(b,a)>0$.
		\item{\textup{(iii)}} For all $t\in \mathbb{N}$, there exists a unique solution for $H_{t}$ in \eqref{eq:def_ONMF}.
		\item{\textup{(iv)}} For all $t\in \mathbb{N}$, the eigenvalues of the positive semidefinite matrix $A_{t}$ that is defined in \eqref{eq:def_ONMF} are at least as large as some constant $\kappa_{1}>0$.
	\end{description}
	Then Algorithm \ref{alg:NDL} with $\mathtt{MCMC}\in \{ \mathtt{Glauber}, \mathtt{Pivot} \}$ for Network Dictionary Learning converges almost surely to the set of local optima of the associated expected loss function. 
\end{corollary}

\begin{proof}
	Let $(X_{t})_{t\in \mathbb{N}}$ be the sequence of $(k^{2}\times N)$ matrices of ``minibatches of subgraph patterns'' $X_{t}$ defined in Algorithm \ref{alg:NDL}. Since Algorithm \ref{alg:NDL} can be viewed as the OMF algorithm \eqref{eq:scheme_online NMF_surrogate2} applied to the dependent sequence $(X_{t})_{ t\in \mathbb{N}}$, it suffices to verify the assumptions \ref{assumption:A1}, \ref{assumption:M1}, and \ref{assumption:M2'} according to Theorem \ref{thm:online NMF_convergence}.

	We first observe that the matrices $X_{t}\in \R_{\ge 0}^{k^{2}\times N}$ computed in line 13 of Algorithm \ref{alg:NDL} do not necessarily form a Markov chain, as the forward evolution of the Markov chain $\x_{t}$ depends not only on the induced $(k\times k)$ matrix $A_{\x_{t}}$, but also on the actual homomorphisms $\x_{t}$. However, note that the `augmented' sequence $\overline{X}_{t}:=(X_{t}, \x_{Nt})$ forms a Markov chain. Indeed, the distribution of $X_{t+1}$ given $X_{t}$ depends only on $\x_{Nt}$ and $A$, since this determines the distribution of the homomorphisms $(\x_{s})_{Nt<s<N(t+1)}$, which in turn determine the $k^{2}\times N$ matrix $X_{t+1}$. 
	
	Under the assumptions (i) and (ii), \citep[Thm 5.7 and 5.8]{lyu2019sampling} shows that the sequence $(\x_{t})_{ t\in \mathbb{N}}$ of homomorphisms $F\rightarrow \G$ is a finite state Markov chain that is irreducible and aperiodic with unique stationary distribution $\pi_{F\rightarrow \G}$ (see \eqref{eq:def_embedding_F_N}). This easily implies the $N$-tuple of homomorphisms $(\x_{s})_{N(t-1)<s<Nt}$ also form a finite-state, irreducible, and aperiodic chain with a unique stationary distribution. Consequently, the Markov chain $(\overline{X}_{t})_{ t\in \mathbb{N}}$ that we defined above is also a finite-state, irreducible, and aperiodic chain with a unique stationary distribution. 
	In this setting, one can regard Algorithm \ref{alg:NDL} as the Online NMF algorithm in \eqref{eq:scheme_online NMF_surrogate2} for the input sequence $X_{t}=\varphi(\overline{X}_{t})$, where $\varphi:(X, \x)\mapsto X$ is the projection on its first coordinate. Because $\overline{X}_{t}$ takes only finitely-many values, the range of $\varphi$ is bounded. This verifies all hypotheses of Theorem \ref{thm:online NMF_convergence}, so the assertion follows.
\end{proof}

\begin{remark}
	\normalfont
	A similar convergence result of Algorithm \ref{alg:NDL} with $\mathtt{MCMC}=\mathtt{ApproxPivot}$ also holds. The precise statement and its proof will be given in \citep{lyu2020dictionary}. 
\end{remark}

\subsection{\edit{Applications of Algorithm \ref{alg:NDL} to real-world networks} }

In this subsection, we apply Algorithms \ref{alg:NDL}  to the following real-world networks:

\begin{enumerate}
	\item \dataset{SNAP Facebook} (\dataset{Facebook}) \citep{leskovec2012learning, grover2016node2vec}: This network has 4,039 nodes and 88,234 edges. This network is 
	a Facebook network that has been used as a benchmark example for edge inference. 
	\item \dataset{arXiv ASTRO-PH} (\dataset{arXiv}) \citep{Leskovec2014SNAP,grover2016node2vec}: This network has 18,722 nodes and 198,110 edges. It is a collaboration network between authors of preprints that were posted in the arXiv in astrophysics. Nodes represent scientists and edges indicate coauthorship relationships.		 
	
	\item \dataset{Homo sapiens PPI} (\dataset{H. sapiens}) \citep{oughtred2019biogrid, grover2016node2vec}: This  network has 19,706 nodes and 390,633 edges. The nodes represent proteins in the organism of Homo sapiens, and edges represent physical interactions between these proteins. 
\end{enumerate}

Let $G=(V,E)$ be any simple graph and let $F=([6],A_{F})$ be the $6$-chain motif. Fix a homomorhism $\x:F\rightarrow G$. Then the corresponding $(6\times 6)$ matrix $A_{\x}$ (defined in \eqref{eq:def_F_patch}) is of the following form:
\begin{align}\label{eq:A_patch_ex}
A_{\x}=
\begin{bmatrix}
0 & 1 & * & * &  * &  * \\
1 & 0 & 1 & * &  * &  * \\
* & 1 &  0 & 1 &  * &  * \\
* & * & 1 & 0 &  1 &  * \\
* & * & * & 1 &  0 &  1 \\
* & * & * & * &  1 &  0 
\end{bmatrix}
,
\end{align}    
where entries marked as $*$ may be 0 or 1 (not necessarily the same values). Notice that the 1's in the diagonal line above the main diagonal correspond to the entries $A(\x(i),\x(i+1))$, $1\le i<6$, that are required to be 1 since $A_{F}(i,i+1)=1$ and $\x:F\rightarrow G$ is a homomorphism (recall that $A$ is binary since $G$ is a simple graph). The same observation holds for any $k$-chain motifs for any $k\ge 2$. Hence the more interesting information is captured by the entries off of the two diagonal lines.

In Figure \ref{fig:NDL_app}, we show $r=25$ network dictionary elements learned from each of the above networks using Algorithm \ref{alg:NDL} with the following parameters: $F=([21],A_{F})$ the $21$-chain motif, $T=100$, $N=100$, and $\lambda=1$. In Figure \ref{fig:NDL_app}, such entries in the the learned dictionary elements reveal distinctive structures of the networks. Namely, most dictionary elements for \dataset{Facebook} show `communities' (blocks of pixels) and `hub nodes' (diagonal entries with the corresponding row and column in black); \dataset{arXiv} show a very few hub nodes, but do exhibit clusters (groups of black pixels), which is reasonable since scientists tend to collaborate often as a team and it is less likely that there is an overly popular scientist (whereas popular users in Facebook networks are natural); In the \dataset{Homo sapiens PPI}, it seems that proteins there hardly form large clusters of mutual interaction.

\begin{figure*}[h]
	\centering
	\includegraphics[width=1 \linewidth]{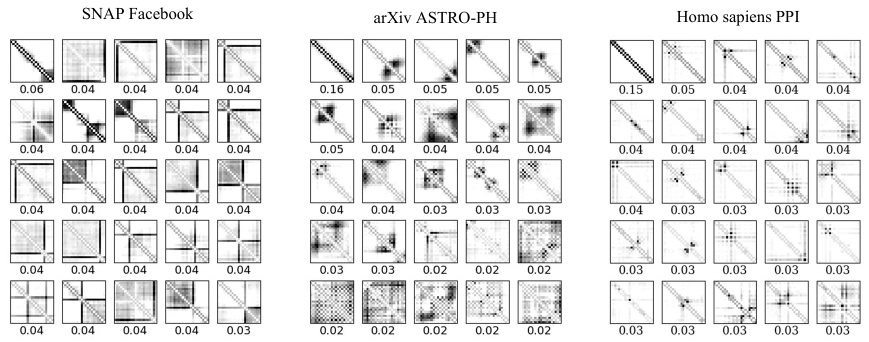}
	\vspace{-0.2cm}
	\caption{$r=25$ learned dictionary elements from networks \dataset{Facebook}, \dataset{arXiv}, and \dataset{H. sapiens} with 21-chain motif $F$. Black=1 and white = 0 with gray scale. The values below the dictionary elements indicate their ``dominance score'', which is computed by the diagonal entries of the $r\times r$ aggregate matrix $P_{t}$ in \eqref{eq:def_ONMF}. 
	}
	\label{fig:NDL_app}
\end{figure*}

Note that in Figure \ref{fig:NDL_app}, we also show the ``dominance score'' for each dictionary element, which has the meaning of its `usage' in approximating a sampled $(k\times k)$ matrix $A_{\x}$ from $G$. It is computed by normalizing the square root of the diagonal entries of the $r\times r$ aggregate matrix $P_{t}$ in \eqref{eq:def_ONMF}. Roughly speaking, the $i$th diagonal entry of $P_{t}$ is approximately the average of the $L_{2}$ norm of the $i$th row of the code matrix $H$ where $X_{t}\approx W H$ for $X_{t}$ the $k^{2}\times N$ matrix of `subgraph patterns' and $W$ the $k^{2}\times r$ dictionary matrix. For instance, in both \dataset{arXiv} and \dataset{H. sapiens}, the `most dominant' dictionary elements (with dominance score over $0.15$) have mostly zeros outside of the two diagonal lines, indicating that these two networks are sparse. 

\subsection{Applications of Algorithm \ref{alg:network_reconstruction} for network denoising problems} In this subsection, we apply Algorithm \ref{alg:network_reconstruction} to solve network denoising problems. Namely, if we are given a simple graph $G_{\textup{true}}=(V,E)$, we may either add some `false edges' or delete some true edges (hence creating `false nonedges') randomly and create a `corrupted version' $G_{\textup{obs}}=(V,E')$ of the original graph $G_{\textup{true}}$. The problem is to recover $G_{\textup{true}}$ when we are only given with the corrupted observed network $G_{\textup{obs}}$. This problem is also known as `network denoising' \citep{correia2019handling} or `edge inference' (or prediction) \citep{liben2007link, lu2011link, menon2011link, kovacs2019network} in the cases of adding or deleting edges, respectively. Here we refer these problems collectively as `network denoising' with `additive noise' or `subtractive noise', correspondingly. Each setting can be regarded as a binary classification problem. Namely, for the additive noise case, it is equal to the binary classification of the edges set $E'$ in the corrupted graph $G_{\textup{obs}}$ into true (positive) and false (negative) edges; for the subtractive noise case, we are to classify the set of all nonedges in $G_{\textup{obs}}$ into true (nonedges in $G_{\textup{true}}$) and false (deleted edges of $G_{\textup{true}}$). 

\begin{figure*}[htbp]
	\centering
	\includegraphics[width=1 \linewidth]{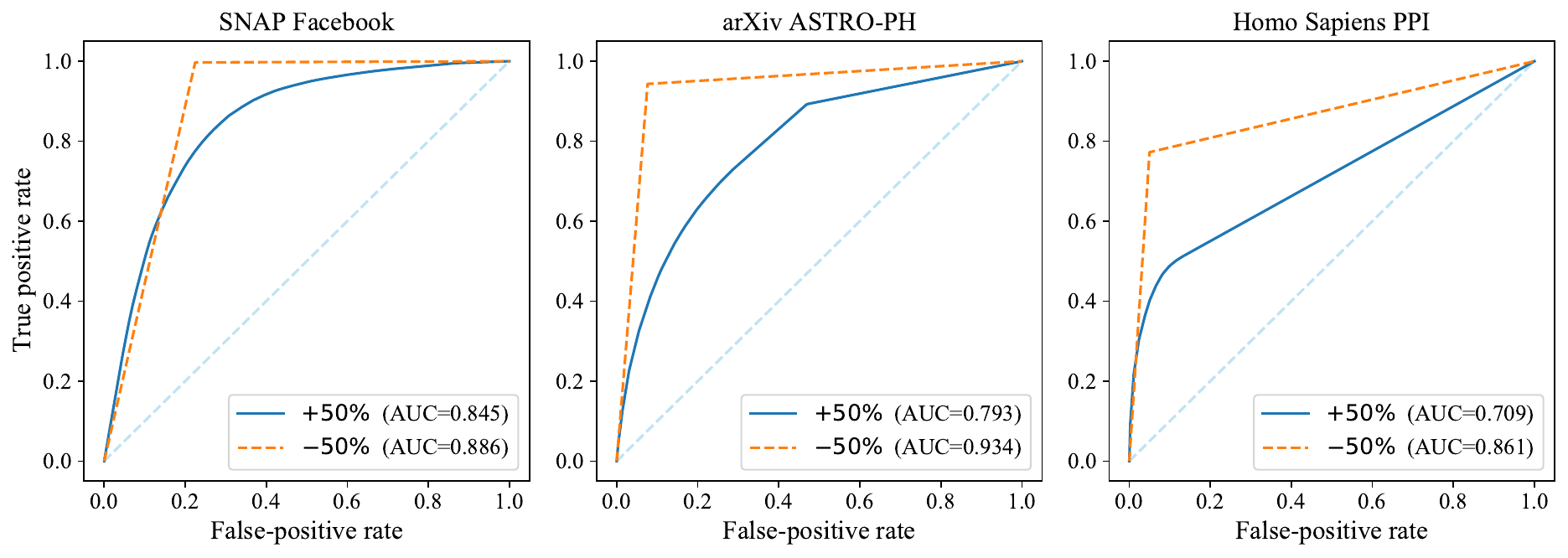}
	\vspace{-0.2cm}
	\caption{
		Application of the NDL and NDR algorithms to network denoising with additive and subtractive noise on Facebook and PPI networks. We first use NDL to learn a network dictionary from a corrupted network and then reconstruct the networks using NDR to assign a confidence value to each potential edge. We then use these confidence values to infer membership of potential edges in the uncorrupted network. 
		Importantly, we never use information from the original networks. For each network, we indicate the receiver-operating characteristic (ROC) curves and corresponding 
		area-under-the-curve (AUC) scores for network denoising with additive noise using the labels $+50$\%, and we indicate the ROC curves and corresponding AUC scores for network denoising with subtractive noise using the labels $-50$\%.
	}
	\label{fig:NR_denoising}
\end{figure*}

To experiment with these problems, we use the three real-world networks: \dataset{Facebook}, \dataset{arXiv}, and \dataset{H. sapiens}.  Given a network $G_{\textup{true}} = (V,E)$, we first generate four corrupted networks $G_{\textup{obs}}$ as follows. In the subtractive noise case, we create a smaller connected network by removing a uniformly chosen random subset that consists of  $50\%$ of the edges from our network. In the additive noise case, we create a corrupted network by adding edges between node pairs independently with a fixed probability so that the new network has $50\%$ new edges.

In order to solve the network denoising problems, we first apply Algorithm \ref{alg:NDL} with $21$-chain motifs with $r=25$ columns in the dictionary matrix to learn a network dictionary for each of these four networks, and we use each dictionary to reconstruct the 
network from which it was learned using Algorithm \ref{alg:network_reconstruction} with parameters $T=200,000$, $\lambda=0$, and $\mathtt{MCMC}=\mathtt{ApproxPivot}$. The reconstruction algorithms output a weighted network $G_{\textup{recons}}=(V,A_{\textup{recons}})$. For denoising additive noise, we classify each edge in the corrupted network as `positive' if its weight in $A_{\textup{recons}}$ is strictly \textit{larger} than some threshold $\theta$. For denoising subtractive noise, we classify each nonedge in the corrupted network as `positive' if its weight in $A_{\textup{recons}}$ is strictly \textit{less} than some threshold $\theta$ (see Remark \ref{remark:NDR} for why we use the opposite directionality of classification for subtractive noise).  By varying $\theta$ we construct a receiver operating characteristic (ROC) curve that consists of points whose horizontal and vertical coordinates are the false-positive rate and true-positive rate, respectively. For instance, if $\theta=0$ in the additive (resp., subtractive) noise case, almost all edges (resp., nonedges) will be classified as `positive' (resp., `negative'), so it will correspond to the corner $(1,1)$ (resp., $(0,0)$) in the ROC curve.

\begin{table}[htbp]
	\centering
	\begin{tabular}{cccc}
		\hline 
		\textit{Algorithm} & \dataset{SNAP Facebook} & \dataset{Homo sapiens PPI} & \dataset{arXiv ASTRO-PH} \\
		\hline 
		Spectral Clustering & 0.619 & 0.492 & 0.574  \\
		DeepWalk & 0.968 & 0.744 & \textbf{0.934} \\
		LINE   & 0.949 & 0.725 & 0.890 \\
		node2vec   & \textbf{0.968} & 0.772 & \textbf{0.934}  \\
		NDL+NDR (our method)   & 0.907 & \textbf{0.861} & \textbf{0.934} \\
		\hline
	\end{tabular}%
	\vspace{0.3cm}
	\caption{Comparison of AUC scores for network denoising with additive ($+50\%$) and subtractive ($-50\%$) noises. The results in the first four rows are obtained from \citet{grover2016node2vec}.}
	\label{table:1}
\end{table}

In Figure \ref{fig:NR_denoising}, we show the ROC curves and corresponding area-under-the-curve (AUC) scores for our network-denoising experiments with subtractive and additive noise for all three networks. For example, if one adds $50\%$ of false edges to the \dataset{Facebook} so that 88,234 edges are true and 44,117 edges are false, then our method achieves AUC of 0.845, and is able to detect over $78\%$ (34,411) of the false edges while misclassifying $20\%$ (17,647) of the true edges (see Figure \ref{fig:NR_denoising} left). In Table \ref{table:1}, we also compare the performance of our method against some popular supervised algorithms based on network embedding, such as node2vec \citep{grover2016node2vec}, DeepWalk \citep{perozzi2014deepwalk}, and LINE \citep{tang2015line} for the task of denoising $50\%$ subtractive noise for \dataset{SNAP Facebook}, \dataset{H. Sapiens}, and \dataset{arXiv}. It is important to note that, unlike these methods, our algorithm for network denoising is \textit{unsupervised} in the sense that it never requires any information from the original network $G$. Nonetheless, our algorithm shows comparable performance and in two cases the best results among all methods considered here. 


\begin{remark}\label{remark:NDR}
	\normalfont
	The reason that we used the opposite directionality of classification for subtractive noise, that is, a nonedge $(p,q)$ in $G_{\textup{obs}}$ is classified `positive' if $A_{\textup{recons}}(p,q)< \theta$, is that we often obtain `flipped' ROC curves using the standard classification scheme
	\begin{align}\label{eq:edge_classification_scheme}
	A_{\textup{recons}}(p,q)>\theta \quad  \Longrightarrow \quad  \text{$(p,q)$ is classified as positive}
	\end{align}
	for denoising subtractive noise. While a complete understanding of this phenomenon is yet to be made, we remark here why the standard classification \eqref{eq:edge_classification_scheme} may not work in our favor for subtractive noise. The issue is related with sparsity of real-world networks, which does not arise in image denoising.

	First recall that every sampled $k\times k$ matrix $A_{\x}$ is conditioned to have 1's on its first super- and sub-diagonals (see \eqref{eq:A_patch_ex}), corresponding to the observed edges of $G_{\textup{obs}}$ in the image of the homomorhpism $\x$. Hence for subtractive noise, false non-edges (deleted edges in $G_{\textup{true}}$) always appear as $0$'s outside the two super- and sub-diagonal lines of $A_{\x}$. Now if $G_{\textup{obs}}$ is sparse (as most real-world networks are), then there are very few positive entries in $A_{\x}$ other than the first super- and sub-diagonal entries. Hence if we approximate such $A_{\x}$ using network dictionary atoms, it is more likely to overfit to reconstruct the observed edges (1's in the first super- and sub-diagonal entries), which will result in small reconstructed weights for the other entries of $A_{\x}$, including the ones corresponding to false non-edges. 
	
	In \cite{lyu2020dictionary}, a modified version of NDR (Algorithm \ref{alg:network_reconstruction}) will be introduced that addresses this issue, so that we can use the classification scheme \eqref{eq:edge_classification_scheme} uniformly both for additive and subtractive noise.

\end{remark}

\color{black}

\vspace{0.2cm}
\section{Proof of Theorem \ref{thm:online NMF_convergence}}

In this section, we provide the proof of our main result, Theorem \ref{thm:online NMF_convergence}.

\subsection{Preliminary bounds}

In this subsection, we derive some key inequalities and preliminary bounds toward the proof of Theorem \ref{thm:online NMF_convergence}.  \edit{Note that proofs are relagated to the appendix.}

\begin{prop}\label{prop:Wt_bd}
	Let $(W_{t}, H_{t})_{t\ge 1}$ be a solution to the optimization problem \eqref{eq:scheme_online NMF_surrogate2}. Then for each $t\in \mathbb{N}$, the following hold almost surely:
	\begin{description}
		\item[(i)] $\hat{f}_{t+1}(W_{t+1}) - \hat{f}_{t}(W_{t}) \le w_{t+1}\left(  \ell(X_{t+1},W_{t}) - f_{t}(W_{t})  \right) = f_{t+1}(W_{t}) - f_{t}(W_{t})$.
		\vspace{0.2cm}
		\item[(ii)] $0\le w_{t+1}\left( \hat{f}_{t}(W_{t}) - f_{t}(W_{t}) \right)  \le    w_{t+1}\left( \ell(X_{t+1},W_{t}) - f_{t}(W_{t}) \right)  + \hat{f}_{t}(W_{t}) - \hat{f}_{t+1}(W_{t+1})$.
	\end{description}
\end{prop}

\begin{proof}
	See Appendix \ref{appendix:proofs}.
\end{proof}

Next, we show that if the data are drawn from compact sets, then the set of all possible codes also form a compact set. 
This also implies boundedness of the matrices $A_{t}\in \R^{r\times r}$ and $B_{t}\in \R^{r\times n}$, which aggregate sufficient statistics up to time $t$ (defined in \eqref{eq:scheme_online NMF_surrogate2}).

The following proposition provides a second-order growth property of the quadratic function for $W_{t}$ in the OMF algorithm \eqref{eq:scheme_online NMF_surrogate2} when the set $\mathcal{C}$ of constraints for the dictionaries is general and not necessarily convex. This is well-known for the convex case (see, e.g., \citep[Lem. B.5]{mairal2013optimization}).

\begin{prop}\label{prop:gW_bd}
	Fix symmetric and positive definite $A\in \mathbb{R}^{r\times r}$, arbitrary $B\in \mathbb{R}^{r\times d}$. Denote $g(W) = \tr(W A W^{T})  - 2\tr(W B)$ for each $W\in \mathbb{R}^{d\times r}$. Then the following hold:
	\begin{description}
		\item[(i)] Let $W_{1},W_{2}\in \mathbb{R}^{d\times r}$ be such that $\tr( (B^{T} - W_{2} A ) (W_{1}-W_{2})^{T}) \le 0$. Then 
		\begin{align}\label{eq:gW_bd_1}
		g(W_{1}) - g(W_{2}) \ge \tr( (W_{1}-W_{2})A(W_{1}-W_{2})^{T}) \ge 0.
		\end{align}
		
		\item[(ii)] Fix $W_{1},W_{2}\in \mathbb{R}^{d\times r}$ and suppose the function $g(\lambda W_{2} +  (1-\lambda )W_{1})$ is  monotone decreasing in $\lambda\in [0,1]$. Then $\tr( (B^{T} - W_{2} A ) (W_{1}-W_{2})^{T}) \le 0$.

		\vspace{0.1cm}
		\item[(iii)] Let $\mathcal{C}\subseteq \mathbb{R}^{d\times r}$ be convex, $W_{t-1}\in \mathcal{C}$ arbitrary, and $W_{t} = \argmin_{W\in \mathcal{C}} g(W)$. Then 
		\begin{align}
		g(W_{t-1}) - g(W_{t}) \ge \tr( (W_{t}-W_{t-1})A(W_{t}-W_{t-1})^{T}) \ge 0.
		\end{align} 

	\end{description}
\end{prop}

\begin{proof}
	See Appendix \ref{appendix:proofs}.
\end{proof}

An important consequence of the above second-order growth condition is an upper bound on the change of learned dictionaries, which is also known as ``iterate stability'' \citep[Lem B.8]{mairal2013stochastic}

\begin{prop}\label{prop:Wt_increment_bd}
	Let $(W_{t}, H_{t})_{t\ge 1}$ be a solution to the OMF scheme \eqref{eq:scheme_online NMF_surrogate2}. Assume \ref{assumption:A1}-\ref{assumption:A2} and \ref{assumption:C2} for \eqref{eq:scheme_online NMF_surrogate2}. Then there exist some constant $c>0$ such that almost surely for all $t\in \mathbb{N}$,
	\begin{align}
	\lVert W_{t+1} - W_{t} \rVert_{F} \le  cw_{t+1}.
	\end{align}	
\end{prop}

\begin{proof}
	See Appendix \ref{appendix:proofs}.
\end{proof}

\begin{remark}
	\textup{
		Proposition \ref{prop:Wt_increment_bd} with triangle inequality shows that $\lVert W_{m}-W_{n} \rVert_{F}\le \sum_{j=n+1}^{m}w_{j}$. Hence if $\sum_{j=1}^{\infty}w_{j}<\infty$, then $W_{t}$ converges in the compact set $\mathcal{C}$ for arbitrary input data sequence $(X_{t})_{ t\in \mathbb{N}}$ in a bounded set. 
	}
\end{remark}

\subsection{Convergence of the empirical and surrogate loss}

We prove Theorem \ref{thm:online NMF_convergence} in this subsection. According to Proposition \ref{prop:Wt_bd}, it is crucial to bound the quantity $\ell(X_{t+1},W_{t}) - f_{t}(W_{t})$. When $Y_{t}$'s are i.i.d., we can condition on the information $\mathcal{F}_{t}$ up to time $t$ so that 
\begin{align}
\E\left[ \ell(X_{t+1},W_{t}) - f_{t}(W_{t}) \,\bigg|\, \mathcal{F}_{t} \right] = f(W_{t}) - f_{t}(W_{t}).
\end{align} 
Note that for each fixed $W\in \mathcal{C}$, $f_{t}(W)\rightarrow f(W)$ almost surely as $t\rightarrow \infty$ by the strong law of large numbers. To handle time dependence of $W_{t}$, one can instead look that the convergence of the supremum $\lVert f_{t} - f \rVert_{\infty}$ over the compact set $\mathcal{C}$, which is provided by the classical Glivenko-Cantelli theorem. This is the approach taken in \citep{mairal2010online, mairal2013stochastic} for i.i.d.\ input. 

However, the same approach breaks down when $(Y_{t})_{ t\in \mathbb{N}}$ is a Markov chain. This is because, conditional on $\mathcal{F}_{t}$, the distribution of $Y_{t+1}$ is not necessarily the stationary distribution $\pi$. Our key innovation to overcome this difficulty is to condition much early on -- at time $t-N$ for some suitable $N=N(t)$. Then the Markov chain runs $N+1$ steps up to time $t+1$, so if $N$ is large enough for the chain to mix, then the distribution of $Y_{t+1}$ conditional on $\mathcal{F}_{t-N}$ is close to the stationary distribution $\pi$. The error of approximating the stationary distribution by the $N+1$ step distribution is controlled using total variation distance and mixing bound.

\begin{prop}\label{prop:increment_bd}
	Suppose \ref{assumption:A1}-\ref{assumption:A2} and \ref{assumption:M1}. Fix $W\in \mathcal{C}$. Then for each $t\in \mathbb{N}$ and $0\le N<t$, conditional on the information $\mathcal{F}_{t-N}$ up to time $t-N$,
	\begin{align}
	\left(\E\left[ \ell(X_{t+1},W) - f_{t}(W)  \,\bigg|\, \mathcal{F}_{t-N} \right] \right)^{+} &\le \left| f(W) - f_{t-N}(W) \right|   + Nw_{t} f_{t-N}(W)  \\
	&\qquad + 2\lVert \ell(\cdot, W) \rVert_{\infty} \sup_{\mathbf{y}\in \Omega} \lVert P^{N+1}(\mathbf{y},\cdot) - \pi \rVert_{TV}.
	\end{align}	
\end{prop}

\begin{proof}
	Recall that for each $s\ge 0$, $\mathcal{F}_{s}$ denotes the $\sigma$-algebra generated by the history of underlying Markov chain $Y_{0},Y_{1},\dots,Y_{s}$.  Fix $\mathbf{y}\in \Omega$ and suppose $Y_{t-N} = \mathbf{y}$. Then by the Markov property, the distribution of $Y_{t+1}$ conditional on $\mathcal{F}_{t-N}$ equals $P^{N+1}(\mathbf{y},\cdot)$, where $P$ denotes the transition kernel of the  chain $(Y_{t})_{ t\in \mathbb{N}}$. Using the fact that $2\lVert \mu-\nu \rVert_{TV} = \sum_{x} |\mu(x)-\nu(x)|$ (see \citep[Prop. 4.2]{levin2017markov}) and recalling $X_{t}=\varphi(Y_{t})$ by \ref{assumption:A1}, it follows that  
	\begin{align}
	\E\left[ \ell(X_{t+1},W)  \,\bigg|\, \mathcal{F}_{t-N} \right] &= \sum_{\mathbf{y}'\in\Omega} \ell(\varphi(\mathbf{y}'),W)\,P^{N+1}(\mathbf{y},\mathbf{y}') \\
	&= \sum_{\mathbf{y}'\in\Omega} \ell(\varphi(\mathbf{y}'),W)\, \pi(\mathbf{y}') + \sum_{\mathbf{y}'\in \Omega} \ell(\varphi(\mathbf{y}'),W)(P^{N+1}(\mathbf{y},\mathbf{y}') - \pi(\mathbf{y}'))  \\
	&\le \sum_{\mathbf{y}'\in \Omega} \ell(\varphi(\mathbf{y}'),W)\,\pi(\mathbf{y}') + 2\lVert \ell(\cdot, W) \rVert_{\infty} \lVert P^{N+1}(\mathbf{y},\cdot) - \pi \rVert_{TV} \\
	&= f(W) + 2\lVert \ell(\cdot, W) \rVert_{\infty} \lVert P^{N+1}(\mathbf{y},\cdot) - \pi \rVert_{TV}.
	\end{align}	
	Also, observe that 	
	{\small
		\begin{align}
		\E\left[ -f_{t}(W)  \,\bigg|\, \mathcal{F}_{t-N} \right] &= -f_{t-N}(W) \prod_{k=t-N+1}^{t}(1-w_{k})   - \E\left[ \sum_{k=t-N+1}^{t} \ell(X_{k},W)w_{k}\prod_{j=k+1}^{t}(1-w_{k})  \,\bigg|\, \mathcal{F}_{t-N}\right] \\
		&\le -f_{t-N}(W) +  f_{t-N}(W)\left(1- \prod_{k=t-N+1}^{t}(1-w_{k}) \right) \le -f_{t-N}(W) + Nw_{t}f_{t-N}(W),
		\end{align}
	}
	where we have used the fact that $\ell \ge 0$ and $w_{k}\in (0,1)$ is non-increasing in $k$. Then combining the two bounds and a triangle inequality give  the assertion. 
\end{proof}

Next, we provide some probabilistic lemmas. 

\begin{lemma}\label{lem:uniform_convergence_symmetric_weights}
	Under the assumptions \ref{assumption:A1}-\ref{assumption:A2} and \ref{assumption:M1}, 
	\begin{align}
	\E\left[ \sup_{W\in \mathcal{C}} \sqrt{t}\left| f(W) - \frac{1}{t} \sum_{s=1}^{t} \ell(X_{s},W)\right| \right] = O(1).
	\end{align}
	Furthermore, $\sup_{W\in \mathcal{C}} \left| f(W) - \frac{1}{t} \sum_{s=1}^{t} \ell(X_{s},W)\right|\rightarrow 0$ almost surely as $t\rightarrow \infty$. 
\end{lemma}

\begin{proof}
	Let $\mathcal{F}$ denote the collection of functions $\ell(\varphi(\cdot), W):\Omega\rightarrow [0,\infty)$ indexed by $W\in \mathcal{C}$, which are bounded and measurable under \ref{assumption:A1}. The underlying Markov chain $(Y_{t})_{ t\in \mathbb{N}}$ has countable state space $\Omega$ and is positive recurrent under \ref{assumption:M1}. Then the second part of the statement is a direct consequence of the uniform SLLN for Markov chains \citep[Thm. 5.8]{levental1988uniform}. For the first part, note that by the uniform functional CLT for Markov chains \citep[Thm 5.9]{levental1988uniform}, the empirical process $\{ g_{t}(W)\,|\, W\in \mathcal{C} \}$, $g_{t}(W):=\sqrt{t} (f(W) - t^{-1} \sum_{s=1}^{t} \ell(\varphi(Y_{s}),W))$ converges weakly to a centered Gaussian process $(\mathbf{X}_{W})_{W\in \mathcal{C}}$ indexed by $\mathcal{C}$ (or equivalently, by $\mathcal{F}$), whose sample paths are bounded and uniformly continuous in the space $\ell^{\infty}(\mathcal{F})$ of bounded functions $\mathcal{F}\rightarrow \mathbb{R}$. Moreover, from the theory of Gaussian processes (see, e.g., \citep{dudley2010sample,talagrand1987regularity}) it is well known that for some universal constant $K>0$,
	\begin{align}
	\E\left[ \sup_{W\in \mathcal{C}} \mathbf{X}_{W} \right] \le K \int_{0}^{\infty} \sqrt{\log N(\eps)} \, \mathrm{d}\eps,
	\end{align}
	where $N(\eps)$ denotes the minimum number of $\eps$-balls needed to cover the parameter space $\mathcal{C}$. Since $\mathcal{C}$ is compact by \ref{assumption:A2}, the right hand side is finite. By the weak convergence of the empirical process, it follows that the expectation in the assertion is uniformly bounded in $t$. This shows the assertion. 
\end{proof}

While the uniform convergence results in Lemma \ref{lem:uniform_convergence_symmetric_weights} applies to empirical loss functions of balanced weights (e.g., $w_{t}=1/t$ for all $t\ge 1$), we may need a similar uniform convergence results for the general weights. The following lemma is due to Mairal \citep[Lem B.7]{mairal2013stochastic}, which originally extended the uniform convergence result to weighted empirical loss functions with respect to i.i.d.\ input data. An identical argument gives the corresponding result in our Markovian case, but we provide it here for the sake of completeness.   

\begin{lemma}\label{lem:uniform_convergence_asymmetric_weights}
	Under the assumptions \ref{assumption:A1}-\ref{assumption:A2} and \ref{assumption:M1}, 
	\begin{align}
	\E\left[ \sup_{W\in \mathcal{C}} \left| f(W) - f_{t}(W)\right| \right] \le C w_{t} \sqrt{t}.
	\end{align}
	Furthermore, if $\sum_{t=1}^{\infty}w_{t}^{2}\sqrt{t}<\infty$, then $\sup_{W\in \mathcal{C}} \left| f(W) - f_{t}(W)\right|\rightarrow 0$ almost surely as $t\rightarrow \infty$. 
\end{lemma}

\begin{proof}
	See Appendix \ref{appendix:proofs}.
\end{proof}

Next, we use the concentration bound in Lemma \ref{lem:uniform_convergence_asymmetric_weights} together with the mixing condition \ref{assumption:M2} to show that the surrogate loss process $(\hat{f}_{t}(W_{t}))_{ t\in \mathbb{N}}$ has the bounded positive expected variation.

\begin{lemma}\label{lemma:increment_bd}
	Let $(W_{t}, H_{t})_{t\ge 1}$ be a solution to the optimization problem \eqref{eq:scheme_online NMF_surrogate2}. Suppose \ref{assumption:A1}-\ref{assumption:A2} and \ref{assumption:M1} hold.
	\begin{description}
		\item[(i)] Let $(a_{t})_{ t\in \mathbb{N}}$ be a sequence of non-decreasing non-negative integers such that $a_{t}\in o(t)$. Then there exists absolute constants $C_{1},C_{2},C_{3}>0$ such that for all sufficiently large $t\in \mathbb{N}$, 
		\begin{align}
		& \mathbb{E}\left[ \left(\E\left[ w_{t+1}\big(\ell(X_{t+1},W_{t}) - f_{t}(W_{t}) \big)\,\bigg|\, \mathcal{F}_{t-a_{t}} \right]\right)^{+} \right] \le C_{1}w_{t-a_{t}}^{2}\sqrt{t} + C_{2}w_{t}^{2}a_{t} \\
		&\hspace{5cm}+ C_{3}w_{t} \sup_{\mathbf{x}\in \Omega} \lVert P^{a_{t}+1}(\mathbf{x},\cdot) - \pi \rVert_{TV}.
		\end{align}	
		
		\item[(ii)] Further assume that \ref{assumption:M2} holds. Then we have 
		\begin{align}\label{eq:lem_increment_2}
		\sum_{t=0}^{\infty}  \left( \E\left[ \hat{f}_{t+1}(W_{t+1}) - \hat{f}_{t}(W_{t}) \right]\right)^{+}  \le \sum_{t=0}^{\infty}  w_{t+1}\left( \E\left[ \left(  \ell(X_{t+1},W_{t}) - f_{t}(W_{t}) \right) \right]\right)^{+} <\infty.
		\end{align}
	\end{description}
\end{lemma}

\begin{proof}
	Recall that $(Y_{t},W_{t})\in \Omega\times \mathcal{C}$ for all $t\in \mathbb{N}$ and both $\Omega$ and $\mathcal{C}$ are compact. Since $\varphi:\Omega\rightarrow \mathbb{R}^{d\times n}$ is bounded, we have 
	\begin{align}\label{def:L}
	L:=\sup_{Y\in \Omega, W\in \mathcal{C}} \ell(\varphi(Y), W) <\infty.
	\end{align}	
	Denote 
	\begin{align}
	\Delta_{t}:=\sup_{\mathbf{y}\in \Omega} \lVert P^{t}(\mathbf{y},\cdot) - \pi \rVert_{TV}.
	\end{align}
	Note that $\lVert f_{s} \rVert_{\infty} \le L$ for any $s\ge 0$. Hence according to Propositions \ref{prop:increment_bd}, we have 
	\begin{align}\label{eq:pf_lem_increment_1}
	\left|\E\left[ w_{t+1}\left( \ell(X_{t+1},W_{t}) - f_{t}(W_{t}) \right) \,\bigg|\, \mathcal{F}_{t-a_{t}} \right] \right| &\le w_{t}\left(\lVert f-f_{t-a_{t}} \rVert_{\infty} \right) + 2L w_{t}^{2}a_{t}  + Lw_{t}\Delta_{t}.
	\end{align} 
	Since $a_{t}=o(t)$, we have $a_{t}\le t/2$ for all sufficiently large $t\in \mathbb{N}$. Then by Lemma \ref{lem:uniform_convergence_asymmetric_weights}, there exists a constant $C_{1}>0$ such that for all sufficiently large $t\in \mathbb{N}$, 
	\begin{align}
	\E\left[\lVert f - f_{t-a_{t}} \rVert_{\infty}\right] \le C_{1}w_{t-a_{t}}\sqrt{t-a_{t}}.
	\end{align}
	Noting that $w_{s}$ is non-increasing in $s$, this gives 
	\begin{align}
	\mathbb{E}\left[ w_{t} \lVert f-f_{t-a_{t}} \rVert_{\infty} \right]& \le C_{1}w_{t-a_{t}}^{2}\sqrt{t}
	\end{align} 
	for all sufficiently large $t\ge 1$. Hence taking expectation on both sides of \eqref{eq:pf_lem_increment_1} with respect to the information from time $t-a_{t}$ to $t$ yields the first assertion.

	Now we show the second assertion. The first inequality in the assertion follows by Proposition \ref{prop:Wt_bd} (i). To show that the last expression is finite, denote $Z_{t}=w_{t+1}^{-1}(\ell(X_{t+1},W) - f_{t}(W))$. Note that by the first assertion and \ref{assumption:M2}, we have   
	\begin{align}
	\sum_{t=1}^{\infty}\E\left[ \left( \E\left[ Z_{t} \,\bigg|\, \mathcal{F}_{t-a_{t}} \right] \right)^{+} \right] <\infty\edit{.}
	\end{align}	
	Then by iterated expectation and Jensen's inequality, it follows that 
	\begin{align}
	\sum_{t=1}^{\infty}(\E[Z_{t}])^{+} &= \sum_{t=1}^{\infty}\left(\E\left[\E\left[Z_{t}\,\bigg|\, \mathcal{F}_{t-a_{t}}\right]\right]\right)^{+} \le \sum_{t=1}^{\infty}\E\left[\left( \E\left[Z_{t}\,\bigg|\, \mathcal{F}_{t-a_{t}}\right]\right)^{+}\right] <\infty.
	\end{align}
	This completes the proof of (ii). 
\end{proof}

\begin{lemma}\label{lemma:main_finite_sum}
	Let $(W_{t}, H_{t})_{t\ge 1}$ be a solution to the optimization problem \eqref{eq:scheme_online NMF_surrogate2}. Suppose \ref{assumption:A1}-\ref{assumption:A2} and \ref{assumption:M1}-\ref{assumption:M2}  hold. Then the following hold. 
	\begin{description}
		\item[(i)] $\E[\hat{f}_{t}(W_{t})]$ converges as $t\rightarrow \infty$.
		\vspace{0.2cm}
		\item[(ii)] $\displaystyle \E\left[ \sum_{t=0}^{\infty} w_{t+1}\left( \hat{f}_{t}(W_{t}) - f_{t}(W_{t}) \right) \right] = \sum_{t=0}^{\infty} w_{t+1}\left(\E[\hat{f}_{t}(W_{t})] - \E[ f_{t}(W_{t})]\right)  <\infty$. 
		
		\item[(iii)] $\displaystyle \sum_{t=0}^{\infty} w_{t+1}\left( \hat{f}_{t}(W_{t}) - f_{t}(W_{t}) \right)  <\infty$ almost surely.
	\end{description}
\end{lemma}

\begin{proof}
	In order to show that $\E[\hat{f}_{t}(W_{t})]$ converges as $t\rightarrow \infty$, since $f_{t}(W_{t})$ is bounded uniformly in $t$, it suffices to show that the sequence $(\E[\hat{f}_{t}(W_{t})])_{t\in \mathbb{N}}$ has a unique limit point. To this end, observe that for any $x,y\in \mathbb{R}$, $(x+y)^{+}\le x^{+}+y^{+}$. Note that, for each $m,n\ge 1$ with $m>n$, 
	\begin{align}
	\left(  \E[\hat{f}_{m}(W_{m})] - \E[\hat{f}_{n}(W_{n})]\right)^{+} &\le \sum_{k=n}^{m-1} \left( \E[\hat{f}_{k+1}(W_{k+1})] - \E[f_{k}(W_{k})]\right)^{+} \\
	&\le \sum_{k=n}^{\infty} \left(  \E[\hat{f}_{k+1}(W_{k+1})] - \E[\hat{f}_{k}(W_{k})]\right)^{+}.
	\end{align}
	The last expression converges to zero as $n\rightarrow \infty$ by Lemma \ref{lemma:increment_bd} (ii). This implies that the sequence $(\E[\hat{f}_{t}(W_{t})])_{t\in \mathbb{N}}$ has a unique limit point, as desired.

	The first equality follows from Fubini's theorem by noting that $\hat{f}_{t}-f_{t}\ge 0$. On the other hand, by using Proposition \ref{prop:Wt_bd} (ii), 
	\begin{align}
	\sum_{t=0}^{\infty} w_{t+1}\left(\E[\hat{f}_{t}(W_{t})] - \E[ f_{t}(W_{t})]\right) &\le \sum_{t=0}^{\infty}  w_{t+1}\left(\E[\ell(X_{t+1},W_{t})] - f_{t}(W_{t})]\right)^{+} \\
	&\qquad \qquad- \sum_{t=0}^{\infty} \left( \E[\hat{f}_{t+1}(W_{t+1})] - \E[\hat{f}_{t}(W_{t})] \right).
	\end{align}
	The first sum on the right hand side is finite by Lemma \ref{lemma:increment_bd} (ii), and the second sum is also finite since we have just shown that $\E[\hat{f}_{t}(W_{t})]$ converges as $t\rightarrow \infty$. This shows (ii). Lastly, recall that non-negative random variable of finite expectation must be finite almost surely. Hence (iii) follows directly from (ii). 
\end{proof}

Now we prove the first main result in this paper, Theorem \ref{thm:online NMF_convergence}.

\begin{proof}[\textbf{Proof of Theorem \ref{thm:online NMF_convergence}}]
	Suppose \ref{assumption:A1}-\ref{assumption:A2} and \ref{assumption:M1}-\ref{assumption:M2} hold. We first show (ii). Recall Lemma \ref{lemma:main_finite_sum} (iii). Both $\hat{f}_{t}$ and $f_{t}$ are uniformly bounded and Lipschitz by Proposition \ref{prop:loss_Lipschitz}. Hence writing $h_{t} = \hat{f}_{t} - f_{t}$, using Proposition \ref{prop:Wt_increment_bd}, there exists a constant $C>0$ such that for all $t\ge 1$,
	\begin{align}\label{eq:pf_thm_1_f_fhat_same_hypothesis}
	|h_{t+1}(W_{t+1}) - h_{t}(W_{t})| &\le |h_{t+1}(W_{t+1})-h_{t+1}(W_{t})|+|h_{t+1}(W_{t})-h_{t}(W_{t})| \\
	&\le C \lVert W_{t+1}-W_{t} \rVert_{F} + \left|\left(\hat{f}_{t+1}(W_{t}) - \hat{f}_{t}(W_{t})\right) -  \left( f_{t+1}(W_{t}) - f_{t}(W_{t})\right) \right| \\
	&= C \lVert W_{t+1}-W_{t} \rVert_{F} + w_{t+1} |\hat{f}_{t}(W_{t}) - f_{t}(W_{t})| = O(w_{t+1}).
	\end{align} 
	Thus, according to Proposition \ref{lem:positive_convergence_lemma}, it follows from Lemma \ref{lemma:main_finite_sum} (ii) that 
	\begin{align}\label{eq:pf_thm_1_f_fhat_same2}
	\lim_{t\rightarrow \infty} \left( \hat{f}_{t}(W_{t}) - f_{t}(W_{t}) \right) = 0 \qquad \text{a.s.} 
	\end{align} 
	Moreover, for all $t\ge 1$, triangle inequality gives 
	\begin{align}
	|f(W_{t}) - \hat{f}_{t}(W_{t})| \le \left( \sup_{W\in \mathcal{C}}| f(W)-f_{t}(W) |\right) + |f_{t}(W_{t}) - \hat{f}_{t}(W_{t})|.
	\end{align}
	The right hand side converges to zero almost surely as $t\rightarrow \infty$ by what we have just shown above and Lemma \ref{lem:uniform_convergence_asymmetric_weights}. This shows (ii).
	
	Next, we show (i). Recall that $\E[\hat{f}_{t}(W_{t})]$ converges by Lemma \ref{lemma:main_finite_sum}. \edit{The Jensen’s inequality and the bounds imply} 
	\begin{align}
	\left|\E[h_{t+1}(W_{t+1})] - \E[h_{t}(W_{t})] \right| \le \E\left[ |h_{t+1}(W_{t+1}) - h_{t}(W_{t})| \right] =  O(w_{t+1}).
	\end{align}
	Since $\E[\hat{f}_{t}(W_{t})]\ge \E[f_{t}(W_{t})]$, Lemma \ref{lemma:main_finite_sum} (i)-(ii) and Lemma \ref{lem:positive_convergence_lemma} give 
	\begin{align}
	\lim_{t\rightarrow \infty} \E[f_{t}(W_{t})] = \lim_{t\rightarrow \infty} \E[\hat{f}_{t}(W_{t})]  +  \lim_{t\rightarrow \infty} \left( \E[f_{t}(W_{t})]-\E[\hat{f}_{t}(W_{t})] \right) = \lim_{t\rightarrow \infty} \E[\hat{f}_{t}(W_{t})] \in (1,\infty).
	\end{align}
	This shows (i).

	Lastly, we show (iii). Denote $g_{t}(W)=\tr(W A_{t} W^{T}) - 2\tr(WB_{t})$ and $\hat{f}_{t}(W) = g_{t}(W) + r_{t}$ (see \eqref{eq:f_hat_t_trace}). Note that $\nabla_{W} g_{t} = 2(W A_{t} - B_{t})$. We will first show 
	\begin{align}\label{eq:pf_thm_gradient}
	\limsup_{t\rightarrow \infty}\lVert \nabla_{W} f(W_{t}) - \nabla_{W} g_{t}(W_{t}) \rVert_{F} = 0.
	\end{align} 
	First choose a subsequence $(t_{k})_{k\ge 0}$ such that $\lVert \nabla_{W} f(W_{t_{k}}) - \nabla_{W} g_{t_{k}}(W_{t_{k}}) \rVert_{F}$ converges. Recall that the sequence $(W_{t}, A_{t},B_{t}, r_{t})_{ t\in \mathbb{N}}$ is bounded by Proposition \ref{prop:H_bdd} and \ref{assumption:A1}-\ref{assumption:A2}. Hence we may choose a further subsequence of $(t_{k})_{k\ge 0}$, which we will denote by \edit{$(s_{k})_{k\in \mathbb{N}}$}, so that $(W_{s_{k}}, A_{s_{k}},B_{s_{k}}, r_{s_{k}})$ converges to some $(W_{\infty}, A_{\infty},B_{\infty}, r_{\infty})$ in $\R^{d\times r}\times \R^{r\times r} \times \R^{r\times n}\times \R$ a.s. as $k\rightarrow \infty$. Define a function 
	\begin{align}\label{eq:g_tilde}
	\hat{f}(W) = \tr(W A_{\infty} W^{T})  - 2\tr(W B_{\infty}) + r_{\infty}.
	\end{align}
	Then we write 
	\begin{align}
	\lVert \nabla_{W} f(W_{s_{k}}) - \nabla_{W} g_{s_{k}}(W_{s_{k}}) \rVert_{F} &\le \lVert\nabla_{W} f(W_{s_{k}}) - \nabla_{W} f(W_{\infty}) \rVert_{F} +  \lVert \nabla_{W} f(W_{\infty}) - \nabla_{W} \hat{f}(W_{\infty}) \rVert_{F} \\
	&\hspace{3cm} + \lVert \nabla_{W} \hat{f}(W_{\infty}) - \nabla_{W} g_{s_{k}}(W_{s_{k}}) \rVert_{F}.
	\end{align}
	By the choice of $(s_{k})_{k\in \mathbb{N}}$, the first term in the right hand side vanishes as $k\rightarrow \infty$. For the second term, note that $\hat{f}_{t}\ge f_{t}$ for all $t\in \mathbb{N}$ and over all $\mathcal{C}$. Hence, for each $W\in \mathcal{C}$, almost surely,
	\begin{align}
	\hat{f}(W) = \lim_{k\rightarrow \infty} \hat{f}_{s_{k}}(W)\ge \lim_{k\rightarrow \infty} f_{s_{k}}(W) = f(W),
	\end{align}
	where the last equality follows from Markov chain ergodic theorem (see, e.g., \citep[Thm 6.2.1, Ex. 6.2.4]{Durrett} or \citep[Thm. 17.1.7]{meyn2012markov}). 
	Moreover, by part (i), we know that 
	\begin{align}
	\hat{f}(W_{\infty})=\lim_{k\rightarrow \infty} \hat{f}_{s_{k}}(W_{s_{k}})= f(W_{\infty})\in (0,\infty)
	\end{align}
	almost surely. Hence by using a Taylor expansion and the fact that $\nabla_{W} f$ is Lipschitz (see \ref{assumption:C1}), it follows that 
	\begin{align}
	\nabla_{W} f(W_{\infty}) = \nabla_{W} \hat{f}(W_{\infty}).
	\end{align}
	For the last term, note that 
	\begin{align}
	\lVert \nabla_{W} \hat{f}(W_{\infty})- \nabla_{W} g_{s_{k}}(W_{s_{k}}) \rVert_{F} = \lVert 2W(A_{\infty}- A_{s_{k}}) - 2(B_{\infty}^{T}-B_{s_{k}}) \rVert_{F} \rightarrow 0,
	\end{align}
	as $A_{s_{k}}\rightarrow A_{\infty}$ and $B_{s_{k}}\rightarrow B_{\infty}$ by the choice of $(s_{k})_{k\ge 0}$. 
	\begin{align}\label{eq:pf_thm_gradient1}
	\limsup_{k\rightarrow \infty}\lVert \nabla_{W} f(W_{s_{k}}) - \nabla_{W} g_{s_{k}}(W_{s_{k}}) \rVert_{F} = 0.
	\end{align} 
	Since $(s_{k})_{k\in \mathbb{N}}$ is a further subsequence of $(t_{k})_{k\ge 0}$ and since $\lVert \nabla_{W} f(W_{t_{k}}) - \nabla_{W} g_{t}(W_{t_{k}}) \rVert_{F}$ converges along $(t_{k})_{k\ge 0}$, the same also holds for $(t_{k})_{k\ge 0}$. This shows \eqref{eq:pf_thm_gradient}.
	
	To conclude that $W_{\infty}$ is a local extremum of $f$, it is enough to show that every limit point of the gradients $\nabla_{W} f(W_{t})$ is in the normal cone of $\mathcal{C}$. Choose a subsequence $(t_{k})_{k\ge 0}$ such that $\nabla_{W}f(W_{t_{k}})$ converges. According to the previous part, this implies that $\nabla_{W}g_{t_{k}}(W_{t_{k}})$ should also converge to the same limit. Recall that $W_{t}$ is the minimizer of the quadratic function $g(W)=\tr(WA_{t}W^{T})-2\tr(WB_{t})$ in $\mathcal{C}\cap \mathcal{E}_{t}$, where $\mathcal{E}_{t}$ denotes the ellipsoid $\tr( (B_{t}^{T}-WA_{t})(W_{t-1}-W))\le 0$. Given a subsequence for $t$, we may take a further subsequence along which  $\mathcal{E}_{t}$ converges to a limiting ellipsoid $\mathcal{E}$ (not necessarily unique over the choice of further subsequences). It follows that, along that further subsequence, $W_{t}$ is the minimizer of $g_{t}$ in some convex part $\mathcal{C}_{i}$ of $\mathcal{C}$ (see \ref{assumption:A2}). This verifies that for all $t\in \mathbb{N}$, $\nabla_{W} g_{t}(W_{t})$ is in the normal cone of the constraint set $\mathcal{C}$ at $W_{t}$ (see., e.g., \citep{boyd2004convex}). Thus $\nabla_{W} \lim_{k\rightarrow \infty}f(W_{t_{k}})$ is in the normal cone of $\mathcal{C}$. Hence $\nabla_{W} \hat{f}(W_{\infty})$ is also in the normal cone of $\mathcal{C}$ at $W_{\infty}$, as desired. This completes the proof of the theorem.
\end{proof}

\section*{Acknowledgement}

HL is partially supported by NSF Grant DMS-2010035 and is grateful for helpful discussions with Yacoub Kureh and Joshua Vendrow for network denoising applications of network dictionary learning. DN is grateful to and was partially supported by NSF CAREER DMS $\#1348721$ and NSF BIGDATA $\#1740325$. LB was supported by the Institute for Advanced Study Charles Simonyi Endowment, ARO YIP award W911NF1910027, NSF CAREER award CCF-1845076, and AFOSR YIP award FA9550-19-1-0026.

\vspace{0.3cm}

\small{

}

\vspace{1cm}
\addresseshere

\newpage

\appendix

\section{Auxiliary proofs and lemmas}

\label{appendix:proofs}

In this appendix, we provide some auxiliary proofs and statements that we use in the proof of Theorem \ref{thm:online NMF_convergence}. Among the results we provide here, the proof of Proposition \ref{prop:Wt_increment_bd} is original  and the rest are due originally due to \citep{mairal2010online, mairal2013stochastic}.

\begin{proof}[\textbf{Proof of Proposition \ref{prop:Wt_bd}}]
	To see the equality in (i), note that 
	\begin{align}
	f_{t+1}(W_{t}) - f_{t}(W_{t}) &= (1-w_{t+1})f_{t}(W_{t}) + w_{t+1}\ell(X_{t+1},W_{t}) - f_{t}(W_{t}) \\
	&= w_{t+1}(\ell(X_{t+1},W_{t})-f_{t}(W_{t})).
	\end{align}
	Next, we have  
	\begin{align}
	\hat{f}_{t+1}(W_{t}) = (1-w_{t+1}) \hat{f}_{t}(W_{t}) + w_{t+1} \ell(X_{t+1},W_{t})
	\end{align}
	for all $t\in \mathbb{N}$, so it follows that 
	\begin{align}
	&\hat{f}_{t+1}(W_{t+1}) - \hat{f}_{t}(W_{t}) \\
	&\qquad =  \hat{f}_{t+1}(W_{t+1}) - \hat{f}_{t+1}(W_{t}) + \hat{f}_{t+1}(W_{t}) - \hat{f}_{t}(W_{t})\\
	&\qquad=\hat{f}_{t+1}(W_{t+1}) - \hat{f}_{t+1}(W_{t}) + (1-w_{t+1})\hat{f}_{t}(W_{t}) + w_{t+1}\ell(X_{t+1},W_{t}) -  \hat{f}_{t}(W_{t}))\\
	&\qquad=\hat{f}_{t+1}(W_{t+1}) - \hat{f}_{t+1}(W_{t}) + w_{t+1}(\ell(X_{t+1},W_{t})-f_{t}(W_{t})) + w_{t+1}(f_{t}(W_{t}) - \hat{f}_{t}(W_{t})).
	\end{align} 
	Then the inequalities in both (i) and (ii) follows by noting that $\hat{f}_{t+1}(W_{t+1})\le \hat{f}_{t+1}(W_{t})$ and $f_{t}\le \hat{f}_{t}$. 
\end{proof}

For each $X\in \mathbb{R}^{d\times n}$, $W\in \mathbb{R}^{d\times r}$, and $H\in \mathbb{R}^{r\times n}$, denote $\ell(X,W,H) = \lVert X - W H \rVert_{F}^{2} + \lambda \lVert H \rVert_{1}$ and define 
\begin{align}\label{eq:def_H}
H^{\textup{opt}}(X,W) &= \argmin_{H\in \mathcal{C}'\subseteq \R^{r\times n}} \ell(X,W,H)\edit{.}
\end{align}
Note that under \ref{assumption:C2}, there exists a unique minimizer of the function in the right hand side, with which we identify as $H^{\textup{opt}}(X,W)$. 

\begin{prop}\label{prop:H_bdd}
	Assume \ref{assumption:A1} and let $R = R(\varphi(\Omega))<\infty$ be as defined in \eqref{eq:def_radius_distance}. Then the following hold:
	\begin{description}
		\item[(i)] For all $X\in \Omega$ and $W\in \mathcal{C}$,
		\begin{align}
		\lVert H^{\textup{opt}}(X,W) \rVert_{F}^{2} \le \lambda^{-2} R^{4}.
		\end{align} 
		\item[(ii)] For any sequence $(X_{t})_{t\ge 1}\subseteq \Omega$ and $(W_{t})_{t\ge 1}\subseteq \mathcal{C}$, define $A_{t}$ and $B_{t}$ recursively as in \eqref{eq:scheme_online NMF_surrogate2}. Then for all $t\ge 1$, we have 
		\begin{align}
		\lVert A_{t} \rVert_{F} \le \lambda^{-2} R^{4}, \qquad \lVert B_{t} \rVert_{F} \le \lambda^{-1} R^{3}\edit{.}
		\end{align}
	\end{description}
\end{prop}

\begin{proof}
	From \eqref{eq:def_H}, we have 
	\begin{align}
	\lambda \lVert H^{\textup{opt}}(X,W) \rVert_{1} \le \inf_{H\in \mathcal{C}'\subseteq  \R^{r\times n}} \left( \lVert X - W H \rVert_{F}^{2} +\lambda \lVert H \rVert_{1} \right) \le  \lVert X \rVert_{F}^{2} \le R^{2}.
	\end{align}
	Note that $\lVert H \rVert_{F}^{2}\le \lVert H \rVert_{1}^{2}$ for any $H$. This yields (i). To get (ii), we observe $\lVert XY \rVert_{F} \le \lVert X \rVert_{F} \lVert Y\rVert_{F}$ from the Cauchy-Schwarz inequality. Then (ii) follows immediately from (i) and triangle inequality. 
\end{proof}

Next, we show the Lipschitz continuity of the loss function $\ell(\cdot,\cdot)$. Since $\Omega$ and $\mathcal{C}$ are both compact, this also implies that $\hat{f}_{t}$ and $f_{t}$ are Lipschitz for all $t\in \mathbb{N}$.

\begin{prop}\label{prop:loss_Lipschitz}
	Suppose \ref{assumption:A1} and \ref{assumption:A2} hold, and let $M=2R(\varphi(\Omega)) +\edit{2} R(\mathcal{C}) R(\varphi(\Omega))^{2}/\lambda$.  Then for each $X_{1},X_{2}\in \Omega$  and $W_{1},W_{2}\in \mathcal{C}$, 
	\begin{align}
	| \ell(X_{1}, W_{1}) - \ell(X_{2} ,W_{2}) | \le  M\left( \lVert X_{1}-X_{2} \rVert_{F} + \lambda^{-1} R(\varphi(\Omega)) \lVert W_{1} - W_{2} \rVert_{F} \right).
	\end{align}
\end{prop}

\begin{proof}
	Fix $X\in \Omega\subseteq \R^{d\times n}$ and $W_{1},W_{2}\in \mathcal{C}$. Denote $H^{*}=H^{\textup{opt}}(X_{2},W_{2})$ and $H_{*}=H^{\textup{opt}}(X_{1},W_{1})$. According to Proposition \ref{prop:H_bdd}, \edit{the Frobenius norm} of $H^{*}$ and $H_{*}$ are uniformly bounded by $R(\varphi(\Omega))^{2}/\lambda$. Note that for any $A,B\in \Omega$, the triangle inequality implies 
	\begin{align}
	\lVert a \rVert^{2}_{F}-\lVert b \rVert^{2}_{F} &= \left( \lVert a \rVert_{F} - \lVert b \rVert_{F} \right)  \left(  \lVert a \rVert_{F} + \lVert b \rVert_{F} \right) \\
	&\le  \lVert a - b \rVert_{F}   \left( \lVert a \rVert_{F} + \lVert b \rVert_{F} \right).
	\end{align}
	Also, the Cauchy-Schwartz inequality, \ref{assumption:A1}-\ref{assumption:A2}, and Proposition \ref{prop:H_bdd} imply 
	\begin{align}
	\lVert X_{1} - W_{1}H^{*} \rVert_{F} + \lVert X_{2} - WH^{*} \rVert_{F} &\le \lVert X_{1} \rVert_{F} + \lVert W_{1}H^{*} \rVert_{F} + \lVert X_{2} \rVert_{F} + \lVert W_{2}H^{*} \rVert_{F} \\ 
	&\le \lVert X_{1} \rVert_{F} + \lVert W_{1} \rVert_{F} \lVert H^{*} \rVert_{F} + \lVert X_{2} \rVert_{F} \lVert H{*} \rVert_{F} \\
	&\le 2R(\varphi(\Omega)) + \edit{2}R(\mathcal{C}) R(\varphi(\Omega))^{2}/\lambda.
	\end{align}
	Denoting $M=2R(\varphi(\Omega)) +2 R(\mathcal{C}) R(\varphi(\Omega))^{2}/\lambda$, we have 
	\begin{align}
	\left| \ell(X_{1},W_{1}) - \ell(X_{2},W_{2})  \right| &\le \left| \left( \lVert X_{1} - W_{1} H^{*} \rVert_{F}^{2} +\lambda \lVert H^{*} \rVert_{1}  \right) - \left(\lVert X_{2} - W_{2} H^{*} \rVert_{F}^{2} +\lambda \lVert H^{*} \rVert_{1}  \right) \right| \\
	&\le M \lVert (X_{1}-X_{2}) + (W_{2} - W_{1}) H^{*} \rVert_{F} \\
	&\le M \left(  \lVert X_{1}-X_{2} \rVert_{F} + \lVert  W_{1}-W_{2} \rVert_{F} \cdot \lVert H^{*} \rVert_{F} \right) \\
	&\le M \left(  \lVert X_{1}-X_{2} \rVert_{F} + \lambda^{-1}R(\varphi(\Omega))^{2} \lVert  W_{1}-W_{2} \rVert_{F} \right).
	\end{align}
	This shows the assertion. 
\end{proof}

\begin{proof}[\textbf{Proof of Proposition \ref{prop:gW_bd}}]
	First note that (iii) follows easily from (i) and (ii). Namely, since $g$ is strictly convex and $\mathcal{C}$ is convex, $W_{t} = \argmin_{W\in \mathcal{C}} g(W)$ is uniquely defined and $g(\lambda W_{t} + (1-\lambda) W_{t-1})$ is monotone decreasing in $\lambda\in [0,1]$. Hence (iii) follows from (i) and (ii). 
	
	In order to show (i) and (ii), we first derive a general second order Taylor expansion for the function $g$. Let $W,W',\bar{W}\in \mathbb{R}^{d\times r}$ be arbitrary. Then a simple calculation gives that
	\begin{align}
	g(W) - g(\bar{W}) &= \tr((W-\bar{W}) A (W-\bar{W})^{T})  + 2\tr( (W-\bar{W})(A \bar{W}^{T} - B)),
	\end{align}
	and also a similar expression for $g(W')$. Writing $\Delta W = W'-W$, we get 
	\begin{align}
	g(W)-g(W') &= \tr((W-\bar{W}) A (W-\bar{W})^{T}) - \tr((W'-\bar{W}) A (W'-\bar{W})^{T})\\
	& \qquad + 2\tr( (W-\bar{W})(A \bar{W}^{T} - B)) - 2\tr( (W'-\bar{W})(A \bar{W}^{T} - B))\\
	&= -2\tr( (W'-\bar{W})A (\Delta W)^{T})  + \tr(\Delta W A (\Delta W)^{T}) \label{eq:g_taylor1}\\
	& \qquad +2\tr( (W-\bar{W})(A \bar{W}^{T} - B)) + 2\tr( (\bar{W}-W')(A \bar{W}^{T} - B)). \label{eq:g_taylor2}
	\end{align}
	
	To show (i), let $W_{1},W_{2}\in \mathbb{R}^{d\times r}$ be such that $\tr( (B^{T} - W_{2} A ) (W_{1}-W_{2})^{T}) \le 0$. By taking $W=W_{1}$, $W'=W_{2}$, and $\bar{W}=(A^{-1}B)^{T}$ in the above expansion, we get 
	\begin{align}
	g(W_{1})-g(W_{2}) &= 2\tr( (\bar{W}-W_{2})A (\Delta W)^{T})  + \tr(\Delta W A (\Delta W)^{T}) \\
	&= \tr( (B^{T} - W_{2}A)  (W_{2}-W_{1})^{T}) + \tr(\Delta W A (\Delta W)^{T}) \ge \tr(\Delta W A (\Delta W)^{T}).
	\end{align}
	This shows \eqref{eq:gW_bd_1}, as desired.

	To show (ii), fix $W_{1},W_{2}\in \mathbb{R}^{d\times r}$ and $W^{(\lambda)} =\lambda W_{2} + (1-\lambda)W_{1}$ for $\lambda\in [0,1]$. Then by taking $W=\bar{W} = W^{(\lambda)}$ and $W'=W_{2}$ in \eqref{eq:g_taylor1}-\eqref{eq:g_taylor2}, we get 
	\begin{align}
	g(W^{(\lambda)})-g(W_{2}) &= -\tr( \Delta W A (\Delta W)^{T}) + 2\tr( (W_{\lambda}-W_{2})(A W_{\lambda}^{T} - B)).
	\end{align}
	Suppose $g(W^{(\lambda)})$ is monotone decreasing in $\lambda\in [0,1]$. In particular, $g(W^{(\lambda)})\ge g(W_{2})$. Since $A$ is positive definite, the above equation gives  
	\begin{align}
	2\tr( (W_{\lambda}-W_{2})(A W_{\lambda}^{T} - B)) \le -\tr( \Delta W A (\Delta W)^{T}) \le 0.
	\end{align}
	This yields, for all $\lambda\in [0,1)$, 
	\begin{align}
	\tr( (W_{1}-W_{2})(A W_{\lambda}^{T} - B)) \le 0.
	\end{align}
	By letting $\lambda \nearrow 1$, this gives the desired inequality 
	\begin{align}
	\tr( (B^{T} - W_{\lambda} A ) (W_{1}-W_{2})^{T})   = \tr( (W_{1}-W_{2})(A W_{2}^{T} - B)) \le 0.
	\end{align}
\end{proof}

\begin{proof}[\textbf{Proof of Proposition \ref{prop:Wt_increment_bd}}]
	The argument is almost the same as that of \citep[Lem.1]{mairal2010online}. The only difference is that we use Proposition \ref{prop:gW_bd} for a second order growth property for non-convex constraint set $\mathcal{C}$ for the quadratic optimization problem for $W_{t}$ with additional constraint, as in \eqref{eq:scheme_online NMF_surrogate2}.
	
	Let $A_{t}$ and $B_{t}$ be as in \eqref{eq:scheme_online NMF_surrogate2}. Denote $\hat{g}_{t+1}(W) = \tr(WA_{t+1}W^{T}) - 2\tr(W B_{t})$ and $\hat{h}_{t+1}:=\hat{g}_{t} - \hat{g}_{t+1}$. We first claim that there exists a constant $c>0$ such that 
	\begin{align}\label{eq:Wt_increment_claim}
	|\hat{h}_{t+1}(W) - \hat{h}_{t+1}(W') | \le cw_{t+1} \lVert W-W'\rVert_{F}
	\end{align}
	for all $W,W'\in \mathcal{C}$ and $t\in \mathbb{N}$. To see this, we first write 
	\begin{align}
	\hat{h}_{t+1}(W) = \tr( W(A_{t}-A_{t+1})W^{T}  ) - 2\tr(W(B_{t}-B_{t+1})).
	\end{align}	
	The Cauchy-Schwartz inequality yields \edit{$\tr(A^{T}B) = \sum_{i,j} A_{ij} B_{ij} \le  \lVert A \rVert_{F} \lVert B\rVert_{F}$}, so we have 
	\begin{align}
	\lVert \hat{h}_{t+1}(W)-\hat{h}_{t+1}(W') \rVert_{F} &\le \left| \tr\left( (W-W')(A_{t}-A_{t+1})W^{T} \right) \right| + \left| \tr\left( W'(A_{t}-A_{t+1})(W-W')^{T}\right) \right|\\
	& \hspace{3cm} + 2\left|\tr(W-W')(B_{t}-B_{t+1})\right|\\
	&\le 2\left( R(\mathcal{C}) \lVert A_{t}-A_{t+1}  \rVert_{F} +\lVert B_{t}-B_{t+1} \rVert_{F} \right) \cdot \lVert W-W' \rVert_{F}, 
	\end{align}
	where $R(\mathcal{C})=\sup_{W\in \mathcal{C}} \lVert W \rVert_{F}<\infty$ by \ref{assumption:A2}. Note that $\lVert H_{t} - H^{\textup{opt}}(X_{t},W_{t-1}) \rVert_{F}=O((\log t)^{-2})$ implies that there exists a constant $c_{2}>0$ such that $\lVert H_{t} \rVert_{F}\le \lVert H^{\textup{opt}}(X_{t},W_{t-1})\rVert_{F}+c_{2}$ for all $t\in \mathbb{N}$. Hence by Proposition \ref{prop:H_bdd}, it follows that $\lVert A_{t} \rVert_{F}$ and $\lVert B_{t} \rVert_{F}$ are uniformly bounded in $t$. Thus there exists a constant $C>0$ such that for all $t\in \mathbb{N}$,
	\begin{align}
	\lVert A_{t}-A_{t+1} \rVert_{F}  &=  w_{t+1} \lVert A_{t} - H_{t+1}H_{t+1}^{T} \rVert_{F}  \le Cw_{t+1},
	\end{align}
	and similarly 
	\begin{align}
	\lVert B_{t}-B_{t+1}\rVert_{F} \le C' w_{t+1}
	\end{align}
	for some other constant $C'>0$. Hence the claim \eqref{eq:Wt_increment_claim} follows.

	To finish the proof, according to the assumptions \ref{assumption:A2} and \ref{assumption:C2}, we first apply Proposition \ref{prop:gW_bd} (i) to deduce the following second order growth condition for all $t\in \mathbb{N}$:
	\begin{align}\label{eq:second_order_growth}
	\hat{g}_{t+1}(W_{t}) - \hat{g}_{t+1}(W_{t+1}) \ge \kappa_{1} \lVert W_{t}-W_{t+1} \rVert_{F}^{2} \ge 0.
	\end{align}	
	Using the inequalities $\hat{g}_{t+1}(W_{t+1})\le \hat{g}_{t+1}(W_{t})$ and $\hat{g}_{t}(W_{t})\le \hat{g}_{t}(W_{t+1})$ given by \eqref{eq:second_order_growth}, we deduce 
	\begin{align}
	0&\le \hat{g}_{t+1}(W_{t}) - \hat{g}_{t+1}(W_{t+1}) \\
	&= \hat{g}_{t+1}(W_{t}) - \hat{g}_{t}(W_{t}) + [\hat{g}_{t}(W_{t}) - \hat{g}_{t}(W_{t+1}) ] + \hat{g}_{t}(W_{t+1}) - \hat{g}_{t+1}(W_{t+1}) \\
	&\le  \hat{g}_{t+1}(W_{t}) - \hat{g}_{t}(W_{t}) + \hat{g}_{t}(W_{t+1}) - \hat{g}_{t+1}(W_{t+1}) = \hat{h}_{t+1}(W_{t+1}) - \hat{h}_{t+1}(W_{t}),
	\end{align} 
	Hence by (\ref{eq:second_order_growth}) and the claim \eqref{eq:Wt_increment_claim}, we get 
	\begin{align}
	\kappa_{1}  \lVert W_{t}-W_{t+1} \rVert_{F}^{2} \le \hat{h}_{t+1}(W_{t+1}) - \hat{h}_{t+1}(W_{t}) \le cw_{t+1} \lVert W_{t}-W_{t+1} \rVert_{F}.
	\end{align} 
	This shows the assertion. 
\end{proof}

\begin{proof}[\textbf{Proof of Lemma \ref{lem:uniform_convergence_asymmetric_weights}}]
	Fix $t\in \mathbb{N}$. Recall the weighted empirical loss $f_{t}(W)$ defined recursively using the weights $(w_{s})_{s\ge 0}$ in \eqref{eq:def_loss_expected_empirical}. For each $0\le s \le t$, denote 
	\begin{align}\label{eq:weights_simplified}
	w_{s}^{t} = w_{s}\prod_{j=s}^{t}(1-w_{j}). 
	\end{align}
	Then for each $t\in \mathbb{N}$, we can write 
	\begin{align}\label{eq:ft_simplified}
	f_{t}(W) = \sum_{s=1}^{t} \ell(X_{s},W) w_{s}^{t}
	\end{align}
	Moreover, note that $w_{1}^{t},\dots,w_{t}^{t}>0$ and $w_{1}^{t}+\dots+w_{t}^{t}=1$. Define $F_{i}(W) = (t-i+1)^{-1}\sum_{j=1}^{t}\ell(X_{i}, W)$ for each $1\le i \le t$. By Lemma \ref{lem:uniform_convergence_symmetric_weights}, there exists a constant $c_{1}>0$ such that 
	\begin{align}\label{eq:pf_uniform_convergnece_weighted}
	\E\left[ \sup_{W\in \mathcal{C}} |F_{i}(W) - f(W)| \right] \le \frac{c_{1}}{\sqrt{t-i+1}}
	\end{align} 
	for all $1\le i \le t$. Noting that \edit{$(w^{t}_{1},\dots,w^{t}_{t})$ is a probability distribution on $\{1,\dots,t\}$}, a simple calculation shows the following important identity 
	\begin{align}
	f_{t} - f = \sum_{i=1}^{t} (w^{t}_{i} - w^{t}_{i-1}) (t-i+1) (F_{i} - f),
	\end{align} 
	with the convention of $w^{t}_{0}=0$. Now by triangle inequality \eqref{eq:pf_uniform_convergnece_weighted}, 
	\begin{align}\label{eq:pf_uniform_convergnece_weighted2}
	\E\left[ \sup_{W\in \mathcal{C}} |f_{t}(W) - f(W)| \right] &\le \E\left[ \sum_{i=1}^{t} (w^{t}_{i} - w^{t}_{i-1}) (t-i+1) \sup_{W\in \mathcal{C}} \left|F_{i}(W) - f(W) \right|  \right] \\
	&= \sum_{i=1}^{t} (w^{t}_{i} - w^{t}_{i-1}) (t-i+1)  \E\left[ \sup_{W\in \mathcal{C}} \left|F_{i}(W) - f(W) \right|  \right] \\
	&\le \sum_{i=1}^{t} (w^{t}_{i} - w^{t}_{i-1}) c_{1}\sqrt{t-i+1} \\
	&\le c_{1}\sqrt{t}\sum_{i=1}^{t} (w^{t}_{i} - w^{t}_{i-1}) = c_{1}\sqrt{t} w^{t}_{t}.
	\end{align} 
	Noting that $w^{t}_{t}=w_{t}$ in \eqref{eq:weights_simplified}, this shows the first part of assertion.  We can show the part by using Lemma \ref{lem:positive_convergence_lemma}, following the argument in the proof of \citep[Lem. B7]{mairal2013stochastic}. See the reference for more details. 
\end{proof}

The following deterministic statement on converging sequences is due to Mairal et al. \citep{mairal2010online}.

\begin{lemma}\label{lem:positive_convergence_lemma}
	Let \edit{$(a_{n})_{n\in \mathbb{N}}$, $(b_{n})_{n\in \mathbb{N}}$, and $(c_{n})_{n\in \mathbb{N}}$} be non-negative real sequences such that 
	\begin{align}
	\sum_{n=0}^{\infty} a_{n} = \infty, \qquad \sum_{n=0}^{\infty} a_{n}b_{n} <\infty, \qquad |b_{n+1}-b_{n}|=O(a_{n}).
	\end{align}
	Then $\lim_{n\rightarrow \infty}b_{n} = 0$. 
\end{lemma}

\begin{proof}
	See \citep[Lem. A.5]{mairal2013stochastic}.
\end{proof}

\section{Algorithm for the generalized online NMF scheme}\label{subsection:algorithm}


In this section, we state an algorithm for the generalized online NMF scheme \eqref{eq:scheme_online NMF_surrogate2}. We denote by $\Pi_{S}:\R^{p\times q}\rightarrow S\subseteq \R^{p\times q}$ the projection onto $S$. The main algorithm, Algorithm \ref{algorithm:online NMF} below, is a direct implementation of \eqref{eq:scheme_online NMF_surrogate2}. \edit{In Algorithm \ref{algorithm:dictionary_update}, $A\sqcup B$ denotes the disjoint union for sets $A$ and $B$. }

\begin{algorithm}[H]
	\footnotesize
	\caption{Online NMF for Markovian data}
	\label{algorithm:online NMF}
	\begin{algorithmic}[1]
		\State \textbf{Variables:} 
		\State  \qquad $X_{t}\in \Omega\subseteq  \mathbb{Q}^{d\times n}$: data matrix at time $t\in \mathbb{N}$ 
		\State \qquad $W_{t-1}\in  \mathcal{C}\subseteq \R^{d\times r}$: learned dictionary at time $t$
		\State  \qquad $(A_{t-1},B_{t-1})\in \R^{r\times r}\times \R^{r\times d}$: aggregate sufficient statistic up to time $t$
		\State \qquad $\lambda,\kappa_{1},K>0$: parameters
		\State \qquad $\mathcal{C}'\subseteq \mathbb{R}^{r\times n}$: constraint set of codes
		
		
		\State \textbf{Upon arrival of $\X_{t}$:}
		\State \qquad Compute $H_{t}$ using Algorithm \ref{algorithm:spaser_coding}.
		\State \qquad $A_{t} \leftarrow t^{-1}((t-1)A_{t-1}+H_{t}H_{t}^{T})$, $B_{t} \leftarrow t^{-1}((t-1)B_{t-1}+H_{t}X_{t}^{T})$.
		\State \qquad Compute $W_{t}$ using Algorithm \ref{algorithm:dictionary_update}, with $W_{t-1}$ as a warm restart, so that 
		\begin{align}
		W_{t} = \argmin_{W\in \mathcal{C} \subseteq \R^{d\times r}} \left(  \tr(W A_{t} W^{T})  - 2\tr(W B_{t})\right) \\
		\hspace{0.85cm} \text{s.t.} \,\, \tr( (B_{t}^{T} - WA_{t} ) (W_{t-1}-W)^{T}) \le 0
		\end{align}
		
	\end{algorithmic}
\end{algorithm}

\noindent
\begin{minipage}[t]{\dimexpr.5\textwidth-.5\columnsep}

\vspace{-0.6cm}
\begin{algorithm}[H]
	\footnotesize
	\caption{Sparse coding}
	\label{algorithm:spaser_coding}
	\begin{algorithmic}[1]
		\State \textbf{Variables:} 
		\State  \qquad $X_{t}\in \Omega\subseteq  \mathbb{Q}^{d\times n}$: data matrix at time $t\in \mathbb{N}$ 
		\State \qquad $W_{t-1}\in  \mathcal{C}\subseteq \R^{d\times r}$: learned dictionary at time $t$
		\State \qquad $\lambda>0$: sparsity regularizer
		\State \qquad $\mathcal{C}'\subseteq \mathbb{R}^{r\times n}$: constraint set of codes
		\State \qquad \edit{$J\in  \mathbb{R}^{r\times n}$:} All ones matrix
		
		\State \textbf{Repeat until convergence:}
		\State \qquad \textbf{Do}
		\begin{align}\label{eq:algorithm_H}	
		H_{t}\leftarrow \Pi_{\mathcal{C}'}\left( H_{t} - \frac{W^{T}_{t-1}W_{t-1} H_{t} - W^{T}_{t-1}X_{t} + \lambda J}{\tr(W^{T}_{t-1}W_{t-1})}  \right)
		\end{align}			
		
		\State \textbf{Return $H_{t}$}
	\end{algorithmic}
\end{algorithm}

\end{minipage}
\hfill
\noindent
\begin{minipage}[t]{\dimexpr.5\textwidth-.5\columnsep}
\vspace{-0.6cm}
\begin{algorithm}[H]
	\footnotesize
	\caption{Dictionary update}
	\label{algorithm:dictionary_update}
	\begin{algorithmic}[1]
		\State \textbf{Variables:} 
		\State \qquad $W_{t-1}\in \mathcal{C}=\mathcal{C}_{1}\sqcup \cdots \sqcup \mathcal{C}_{m}\subseteq \R^{d\times r}$: learned dictionary at time $t$
		\State  \qquad $(A_{t},B_{t})\in \R^{r\times r}\times \R^{r\times d}$: aggregate sufficient statistic up to time $t$
		\State \qquad $\lambda,\kappa_{1}>0$: parameters
		\State \textbf{Do} $\mathcal{E}_{t}\leftarrow \{ W\in \mathbb{R}^{r\times d}\,|\, \tr( (B_{t}^{T} - W_{t}A_{t} ) (W_{t-1}-W_{t})^{T}) \le 0 \}$
		\State \textbf{For $i=1,2,\cdots,m$:}
		\State \qquad\textbf{Do $W_{t}^{(i)}\leftarrow W_{t-1}$ and $W_{t}\leftarrow W_{t-1}$}
		\State \qquad \textbf{If $\mathcal{C}_{i}\cap \mathcal{E}_{t}\ne \emptyset$} \textbf{Repeat until convergence:} 
		\State \qquad \qquad \textbf{For $j=1$ to $r$:}
		\begin{align}\label{eq:dictioanry_column_update}
		[W_{t}]_{\bullet j} \leftarrow \Pi_{\mathcal{C}_{i}\cap \mathcal{E}_{t}} \left( [W_{t-1}]_{\bullet j}  - \frac{W_{t-1} [A_{t}]_{\bullet j} - [B_{t}^{T}]_{\bullet j}}{[A_{t}]_{jj}+1}  \right)
		\end{align}
		\State \qquad \textbf{Do $W_{t}^{(i)}\leftarrow W_{t}$}
		\State \textbf{Return} \text{$W_{t} =\underset{W\in \{W_{t}^{(1)},\cdots, W_{t}^{(m)}\}}{\argmin}  \tr(W A_{t} W^{T}) - 2\tr(WB_{t})$}
	\end{algorithmic}
\end{algorithm}
\end{minipage}

\vspace{0.2cm}

When $\mathcal{C}$ is convex, Algorithm \ref{algorithm:dictionary_update} reduces to the dictionary update algorithm in \citep{mairal2010online}, as the ellipsoidal condition in \eqref{eq:scheme_online NMF_surrogate2} becomes redundant (see Proposition \ref{prop:gW_bd} (iii)). We also remark that the specific coordinate descent algorithms \eqref{eq:algorithm_H} and \eqref{eq:dictioanry_column_update} can be replaced by any other standard algorithms such as LARS.

\section{\edit{Algorithms for Network Dictionary Learning and Reconstruction}}

In this section, we provide algorithms for network dictionary learning (NDL) (Algorithm \ref{alg:NDL}) and network reconstruction (Algorithm \ref{alg:network_reconstruction}) as well as MCMC motif sampling algorithms (Algorithms \ref{alg:glauber}, \ref{alg:pivot}) that sample a random homomorphism $\x:F\rightarrow \G$ from a motif $F=([k],A_{F})$ into a network $\G=(V,A)$. We write $\mathtt{vec}(\cdot)$ for the vectorization operator in lexicographic ordering of entries

\begin{algorithm}
	\footnotesize
	\renewcommand{\thealgorithm}{A3}
	\begin{algorithmic}[1]
		\caption{\!. Rejection Sampling of Homomorphisms}\label{alg:rejection_motif}
		\State \textbf{Input:} Network $\G=(V,A)$, motif $F=([k],A_{F})$
		\vspace{0.1cm}
		\State \textbf{Requirement:} There exists at least one homomorphism $F\rightarrow \G$
		
		\vspace{0.1cm}
		\State \textbf{Repeat:} Sample $\x=[\x(1),\x(2),\ldots,\x(k)]\in V^{[k]}$ so that $\x(i)$'s are independent and identically distributed
		
		\State \qquad \textbf{If}  $\prod_{1\le i,j\le k} A(\mathbf{x}(i), \mathbf{x}(j))^{A_{F}(i,j)}>0$:
		\State \qquad \qquad \textbf{Return} $\x:F\rightarrow \G$ and \textbf{Terminate}
		\State \textbf{Output:} A homomorphism $\x:F\rightarrow \G$
	\end{algorithmic}
\end{algorithm}

\begin{algorithm}
	\footnotesize
	\renewcommand{\thealgorithm}{MG}
	\begin{algorithmic}[1]
		\caption{\!. Glauber Chain Update} \label{alg:glauber}
		\State \textbf{Input:} Network $\G=(V,A)$, $k$-chain motif $F=([k],A_{F})$, and homomorphism $\mathbf{x}:F\rightarrow \G$
		
		\vspace{0.1cm}
		\State \textbf{Do:} Sample $v\in [k]$ uniformly at random 
		\State \qquad Sample $\ell\in V$ at random from the distribution $p$ given by 
		\begin{align}\label{eq:marginal_distribution_glauber}
		p(w) = \frac{1}{Z}\left( \prod_{u\in [k]} A(\mathbf{x}(u), w)^{A_{F}(u,v)} \right) \left( \prod_{u\in [k]} A(w, \mathbf{x}(u))^{A_{F}(v,u)} \right), \qquad w \in V
		\end{align}
		\qquad where $Z=\sum_{c\in V} \left(\prod_{u\in [k]} A(\mathbf{x}(u), c)^{A_{F}(u,v)} \right) \left( \prod_{u\in [k]} A(c, \mathbf{x}(u))^{A_{F}(v,u)}\right)$ is the normalization constant.
		\State \qquad Define a new homomorphism $\mathbf{x}':F\rightarrow \G$ by $\mathbf{x}'(w)=\ell$ if $w=v$ and $\mathbf{x}'(w)=\mathbf{x}(w)$ otherwise 
		
		\State \textbf{Output:} Homomorphism $\mathbf{x}':F\rightarrow \G$

	\end{algorithmic}
\end{algorithm}

\begin{algorithm}
	\footnotesize
	\renewcommand{\thealgorithm}{MP}
	\begin{algorithmic}[1]
		\caption{\!. Pivot Chain Update} \label{alg:pivot}

		\State \textbf{Input:} Symmetric network $\G=(V,A)$, motif $F=([k],A_{F})$, and homomorphism $\mathbf{x}:F\rightarrow \G$
		\vspace{0.1cm}
		
		\State \textbf{Parameters:} $\mathtt{AcceptProb}\in \{\mathtt{Exact}, \mathtt{Approximate}\}$

		\vspace{0.1cm}
		\State \textbf{Do:}  $\x'\leftarrow \x$ 
		\State \qquad \textbf{If} $\sum_{c\in V}  A(\x(1),c)=0$: \textbf{Terminate}   
		
		\State \qquad \textbf{Else}:
		\State \quad \qquad Sample $\ell\in V$ at random from the distribution $p_{1}$ given by 
		\begin{align}\label{eq:RWkernel_G}
		p_{1}(w) = \frac{ A(\x(1),w) }{ \sum_{c\in V}  A(\x(1),c)  }\,, \qquad w\in V
		\end{align}
		
		\State \quad\qquad Compute the acceptance probability $\lambda\in [0,1]$ by 
		\begin{align}\label{eq:pivot_chain_acceptance_prob}
		\lambda \leftarrow  
		\begin{cases}
		\left[ \frac{ \sum_{c\in V} A^{k-1}(\ell,c) }{ \sum_{c\in V} A^{k-1}(\x(1),c) }\frac{\sum_{c\in V}  A(c,\x(1))}{\sum_{c\in V}  A(\x(1),c)} \land 1\right] & \text{If $\mathtt{AcceptProb}=\mathtt{Exact}$} \\[10pt]
		\left[\frac{\sum_{c\in V}  A(c,\x(1))}{\sum_{c\in V}  A(\x(1),c)} \land 1\right] &  \text{If $\mathtt{AcceptProb}= \mathtt{Approximate}$} 
		\end{cases}
		\end{align}

		\State \quad\qquad Sample $U\in [0,1]$ uniformly at random, independently of everything else
		
		\State \quad\qquad $\ell\leftarrow \x(1)$ if $U>\lambda$ and $\x'(1)\leftarrow \ell$. 
		
		\State \quad\qquad \textbf{For $i=2,3,\cdots,k$}:
		\State \quad\qquad \quad Sample $\x'(i)\in V$ from the distribution $p_{i}$ given by 
		\begin{align}\label{eq:pivot_conditional}
		p_{i}(w) = \frac{A(\x(i-1), w)}{\sum_{c\in V} A(\x(i-1), c)}\,, \qquad w\in V
		\end{align}

		\State \textbf{Output:} Homomorphism $\mathbf{x}':F\rightarrow \G$
	\end{algorithmic}
\end{algorithm}

\begin{algorithm}[H]
	\renewcommand{\thealgorithm}{NDL}
	\begin{algorithmic}[1]
		\caption{\!. Network Dictionary Learning (NDL)}\label{alg:NDL}
		
		\State \textbf{Input:} Network $\G=(V,A)$ 
		
		\vspace{0.1cm}
		\State \textbf{Parameters:}  \edit{$F=([k],A_{F})$ (motif)\,,\, $T\in \mathbb{N}$ ($\#$ iterations)\,,\, $N\in \mathbb{N}$ ($\#$ homomorphisms per iteration)\,,\, $r\in \mathbb{N}$ ($\#$ latent motifs)\,, \, $\lambda\ge 0$ ($\ell_{1}$-regularizer)}
		
		\vspace{0.1cm}
		\State \edit{\textbf{Options:} $\mathtt{MCMC}\in \{\mathtt{Pivot},\, \mathtt{PivotApprox},\, \mathtt{Glauber} \}$ }

		\vspace{0.1cm}
		\State \textbf{Requirement:} There exists at least one homomorphism $F\rightarrow \G$

		\vspace{0.1cm}
		\State \textbf{Initialization:} 
		\State \quad Sample a homomorphism $\x:F\rightarrow \G$ by the rejection sampling in Algorithm \ref{alg:rejection_motif} 
		\vspace{0.1cm}
		\State \quad $W=$ $(k^{2}\times r)$ matrix of independent entries that we sample uniformly from $[0,1]$
		\State \quad $P_{0}=$ zero matrix of size $r\times r$; $Q_{0}=$ zero matrix of size $r\times k^{2}$
		\vspace{0.1cm}
		\State \textbf{For $t=1,2,\ldots,T$:}
		\vspace{0.1cm}
		\State \quad \textit{MCMC update and minibatch extraction}:
		\State \qquad Successively generate $N$ homomorphisms $\x_{N(t-1)+1}, \x_{N(t-1)+2},\dots, \x_{Nt}$ by applying
		\begin{align}
		\text{Algorithm \ref{alg:pivot} with $\mathtt{AcceptProb}=\mathtt{Exact}$}  &\qquad \text{if $\mathtt{MCMC}=\mathtt{Pivot}$} \\ 
		\text{Algorithm \ref{alg:pivot} with $\mathtt{AcceptProb}=\mathtt{Approximate}$}  &\qquad \text{if $\mathtt{MCMC}=\mathtt{PivotApprox}$} \\ 
		\text{Algorithm \ref{alg:glauber} with $\mathtt{AcceptProb}=\mathtt{Glauber}$}  &\qquad \text{if $\mathtt{MCMC}=\mathtt{Glauber}$} 
		\end{align}
		
		\State \qquad \textbf{For $N(t-1)<s\le Nt$}: 
		\State \qquad \quad  $A_{\mathbf{x}_{t}}\leftarrow $ $k\times k$ matrix defined by $A_{\mathbf{x}_{t}}(a,b)=A(\x_{t}(a),\x_{t}(b))$ for $1\le a,b\le k$

		\State \qquad  $X_{t}\leftarrow$ $k^{2}\times N$ matrix whose $j$th column is $\mathtt{vec}(A_{\mathbf{x}_{\ell}})$ with  $\ell=N(t-1)+j$
		\Statex \qquad \qquad \qquad ($\mathtt{vec}(\cdot)$ denotes the vectorization operator in lexicographic ordering of entries) 
		
		
		\vspace{0.1cm}
		\State \quad \textit{Single iteration of Online Nonnegative Matrix Factorization}: 
		\begin{align}\label{eq:def_ONMF}
		\hspace{1cm}
		\begin{cases}
		H_{t} \leftarrow \argmin_{H\in \R_{\ge 0}^{r\times N}} \lVert X_{t} - W_{t-1}H \rVert_{F}^{2} + \lambda \lVert H \rVert_{1}\qquad (\text{using Algorithm \ref{algorithm:spaser_coding}}) \\
		P_{t} \leftarrow (1-t^{-1})P_{t-1}+ t^{-1} H_{t}H_{t}^{T} \\
		Q_{t} \leftarrow (1-t^{-1})Q_{t-1}+ t^{-1} H_{t}X_{t}^{T} \\
		W_{t} \leftarrow \argmin_{W\in \mathcal{C}^{\textup{dict}} \subseteq \R_{\ge 0}^{k^{2}\times r}} \left(  \tr(W P_{t} W^{T})  - 2\tr(W Q_{t})\right) \quad (\text{using Algorithm \ref{algorithm:dictionary_update}}),
		\end{cases}
		\end{align}
		\qquad \qquad where $\mathcal{C}^{\textup{dict}} := \{ W\in \R_{\ge 0}^{k^{2}\times r} \,|\, \text{columns of $W$ have Frobenius norm at most $1$}\}$
		\State \textbf{Output:} Network dictionary $W_{T}\in \mathcal{C}^{\textup{dict}} \subseteq \R_{\ge 0}^{k^{2}\times r}$
	\end{algorithmic}
\end{algorithm}

\begin{algorithm}[H]
	\renewcommand{\thealgorithm}{NR}
	\begin{algorithmic}[1]
		\caption{\!. Network Reconstruction (NR)}\label{alg:network_reconstruction}
		
		\State \textbf{Input:} Network $\G=(V,A)$\,,\, and network dictionary $W\in \mathbb{R}_{\ge 0}^{k^{2}\times r}$ 
		\vspace{0.1cm}
		\vspace{0.1cm}
		\State \textbf{Parameters:}  \edit{$F=([k],A_{F})$ (motif)\,,\, $T\in \mathbb{N}$ ($\#$ iterations)\,,\,  $\lambda\ge 0$ ($\ell_{1}$-regularizer)}
		
		\vspace{0.1cm}
		\State \edit{\textbf{Options:} $\mathtt{MCMC}\in \{\mathtt{Pivot},\, \mathtt{PivotApprox},\, \mathtt{Glauber} \}$ }
		
		\vspace{0.1cm}
		\State \textbf{Requirement:} There exists at least one homomorphism $F\rightarrow \G$
		\vspace{0.1cm}
		\State \textbf{Initialization:} 
		\State \quad $A_{\textup{recons}}\,,\,A_{\textup{count}}:V^{2}\rightarrow \{0\}$ (zero matrices) 
		\State \quad Sample a homomorphism $\x_{0}:F\rightarrow \G$ by the rejection sampling in Algorithm \ref{alg:rejection_motif} 
		\vspace{0.1cm}
		\State \textbf{For $t=1,2,\ldots,T$:}
		\vspace{0.1cm}
		\State \quad \textit{MCMC update and mesoscale patch extraction}:
		\State \qquad $\x_{t}\leftarrow$ Updated homomorphism obtained by applying 
		\begin{align}
		\text{Algorithm \ref{alg:pivot} with $\mathtt{AcceptProb}=\mathtt{Exact}$}  &\qquad \text{if $\mathtt{MCMC}=\mathtt{Pivot}$} \\ 
		\text{Algorithm \ref{alg:pivot} with $\mathtt{AcceptProb}=\mathtt{Approximate}$}  &\qquad \text{if $\mathtt{MCMC}=\mathtt{PivotApprox}$} \\ 
		\text{Algorithm \ref{alg:glauber} with $\mathtt{AcceptProb}=\mathtt{Glauber}$}  &\qquad \text{if $\mathtt{MCMC}=\mathtt{Glauber}$} 
		\end{align}
		\State \qquad $A_{\mathbf{x}_{t}}\leftarrow $ $k\times k$ matrix defined by $A_{\mathbf{x}_{t}}(a,b)=A(\x_{t}(a),\x_{t}(b))$ for $1\le a,b\le k$
		\State \qquad  $X_{t}\leftarrow $  $k^{2}\times 1$ matrix obtained by vectorizing $A_{\mathbf{x}_{t}}$
		
		\State \quad \textit{Local reconstruction}:
		\State \qquad  $\widetilde{X}_{t}\leftarrow X_{t}$ and $\widetilde{W}\leftarrow W$ 
		\vspace{0.1cm}
		\State \qquad $H_{t}\leftarrow \argmin\limits_{H\in \mathbb{R}_{\ge 0}^{r\times 1}} ||\widetilde{X}_{t} - \widetilde{W} H||_F^2 + \lambda\lVert H\rVert_1$ and $\hat{X_{t}}\leftarrow \widetilde{W}H_{t}$
		
		\State \label{line:mesoscale_recons} \qquad $\hat{A}_{\mathbf{x}_{t};W}\leftarrow$  $k\times k$ matrix obtained by reshaping the $k^{2}\times 1$ matrix $\hat{X}_{t}$ 
		\vspace{0.1cm}
		\State \quad \textit{Update global reconstruction:}
		\State \qquad \textbf{For} $a,b \in \{1, \dots, k\}$: 
		\begin{align}\label{eq:NR_update_step}
		A_{\textup{count}}(\x_{t}(a),\x_{t}(b))&\leftarrow A_{\textup{count}}(\x_{t}(a),\x_{t}(b))+1\\
		j&\leftarrow A_{\textup{count}}(\x_{t}(a),\x_{t}(b))\\
		A_{\textup{recons}}(\x_{t}(a),\x_{t}(b))&\leftarrow (1-j^{-1})A_{\textup{recons}}(\x_{t}(a),\x_{t}(b)) + j^{-1} \hat{A}_{\x_{t};W}(\x_{t}(a),\x_{t}(b))
		\end{align}
		\State \textbf{Output:} Reconstructed network $\G_{\textup{recons}}=(V, A_{\textup{recons}})$ 
	\end{algorithmic}
\end{algorithm}


\newpage
\section{Additional figures}

In this appendix, we give further examples of the Ising model application discussed in Section \ref{section:Ising}. As shown in Figure \ref{fig:ising_NMF_sup}, the learned dictionary elements at the high temperature $T=5$ are much noiser than the ones corresponding to the lower temperatures. This is reasonable since the Ising spins become less correlated at higher temperatures, so we do not expect there are a few dictionary patches that could approximate the highly random configuration. We remark that this is not an artifact of the Gibbs sampler mixing slowly, as it is well known that it mixes faster at high temperature \citep{lubetzky2012critical}. We also remark that, while our convergence theorem (Theorem \ref{thm:online NMF_convergence}) guarantees that our dictionary patches will almost surely converge to a local optimum even under (functional) Markovian dependence, we do \textit{not} know how effective they are in actually approximating the input sequence. This will depend on the model (e.g., temperature) as well as parameters of the algorithm (patch size, number of dictionaries, regularization, etc.). Moreover, as in the high temperature Ising model, effective dictionary learning may not be possible at all, which is also suggested in Figure \ref{fig:ising_error_plot} right.

\begin{figure*}[h]
	\centering
	\includegraphics[width=0.9 \linewidth]{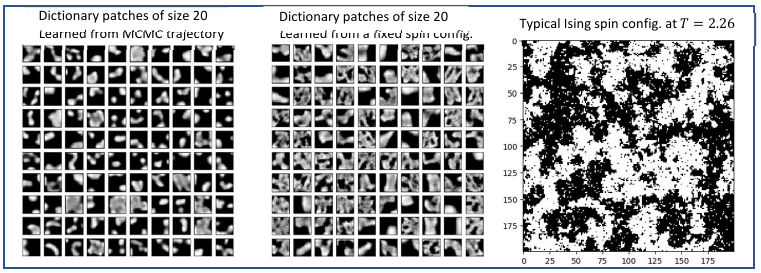}
	\vspace{-0.2cm}
	\caption{
		(Left) 100 learned dictionary patches from a MCMC Gibbs sampler for the Ising model on $200\times 200$ square lattice at a near critical temperature ($T=2.26$). (Middle) 100 learned dictionary patches from fixed Ising spin configuration at $T=0.5$ shown in the right.   
	}
	\label{fig:ising_NMF_cri}
\end{figure*}

\begin{figure*}[h]
	\centering
	\includegraphics[width=0.9 \linewidth]{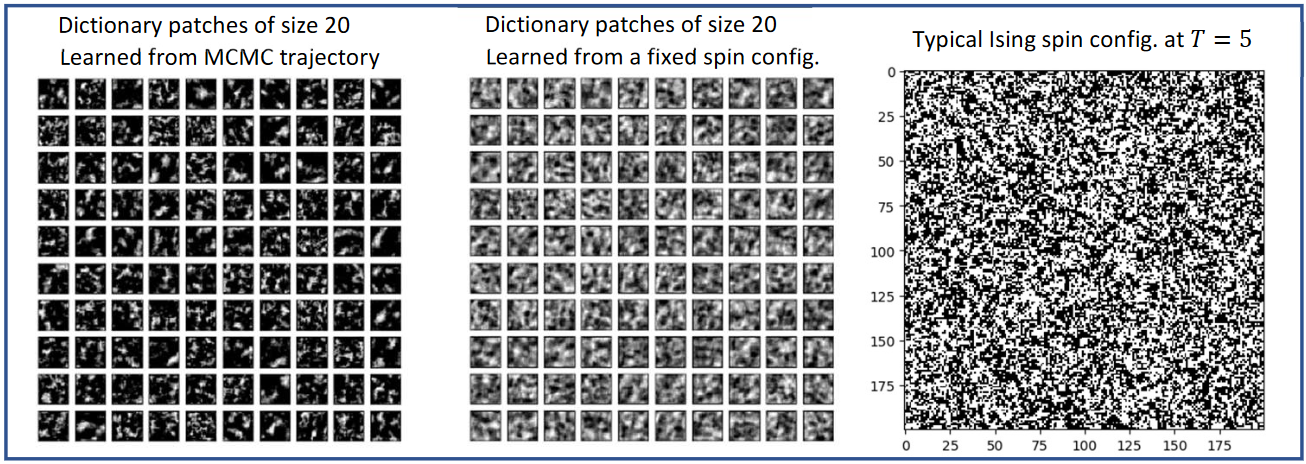}
	\vspace{-0.2cm}
	\caption{
		(Left) 100 learned dictionary patches from a MCMC Gibbs sampler for the Ising model on $200\times 200$ square lattice at a supercritical temperature ($T=5$). (Middle) 100 learned dictionary patches from fixed Ising spin configuration at $T=0.5$ shown in the right.   
	}
	\label{fig:ising_NMF_sup}
\end{figure*}

\begin{figure*}[h]
	\centering
	\includegraphics[width=0.85 \linewidth]{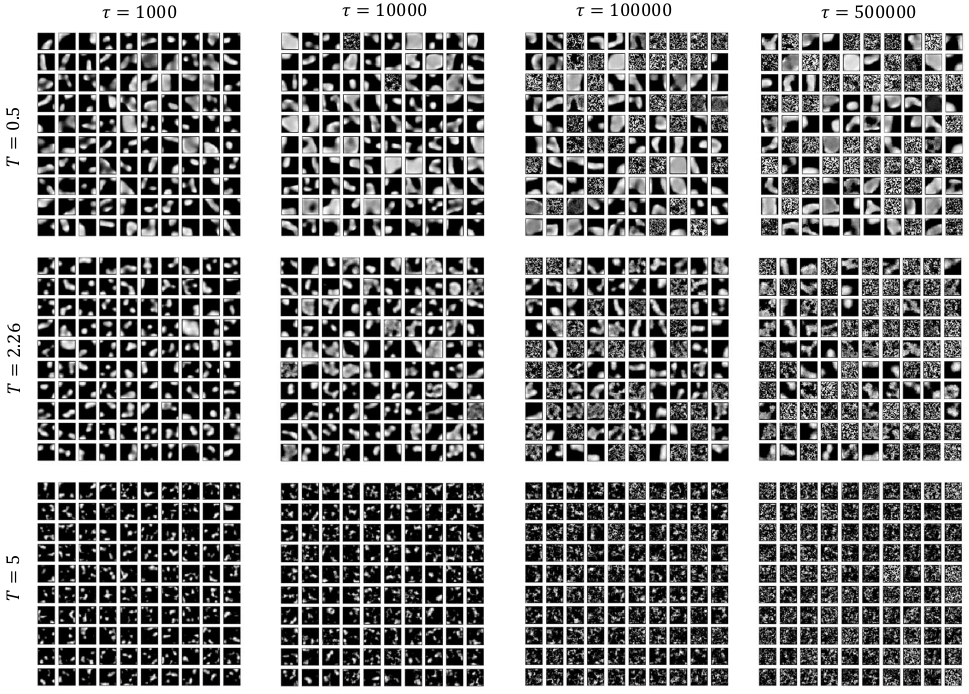}
	\vspace{-0.2cm}
	\caption{
		Plot of (normalized) surrogate errors vs. MCMC iterations (unit $10^4$) for subsampling epochs of 1000, 10000, 100000, and 500000 for temperatures $T=0.5$ (left), $T=2.26$ (middle), and $T=5$ (right), respectively. 
	}
	\label{fig:ising_dicts_subsampling}
\end{figure*}

\begin{figure*}[h]
	\centering
	\includegraphics[width=0.84 \linewidth]{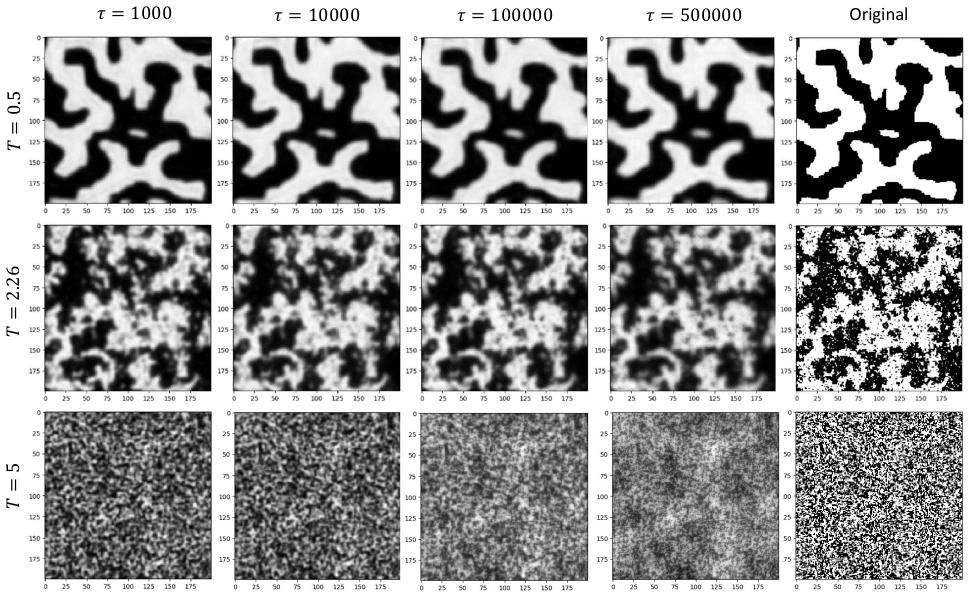}
	\vspace{-0.2cm}
	\caption{
		Reconstruction of fixed Ising spin configurations at temperatures $T=0.5,2.26$, and $5$ (rightmost column) using the learned dictionaries at different subsampling epochs $\tau=1000, 10000, 100000,$ and $500000$ shown in Figure \ref{fig:ising_dicts_subsampling}. Since the dictionaries are learned from the entire MCMC trajectories, reconstruction error of a fixed configuration does not change drastically in the subsampling epoch. 
	}
	\label{fig:ising_dicts_subsampling}
\end{figure*}

\end{document}